\documentclass[10pt]{article}
\usepackage[letterpaper,margin=1in]{geometry}
\usepackage{amsmath,amssymb,amsfonts,amsthm}
\usepackage{mathtools}
\usepackage{graphicx}
\usepackage{epstopdf}
\usepackage{multirow,multicol}
\usepackage{algorithm}
\usepackage{algorithmic}
\usepackage{enumitem}
\usepackage{microtype}
\usepackage[hidelinks]{hyperref}
\usepackage[nameinlink,capitalise,noabbrev]{cleveref}
\DeclareGraphicsExtensions{.pdf,.png,.jpg,.jpeg}

\setlist[enumerate]{leftmargin=.5in}
\setlist[itemize]{leftmargin=.5in}

\newcommand{\email}[1]{\href{mailto:#1}{#1}}

\numberwithin{equation}{section}
\theoremstyle{plain}
\newtheorem{theorem}{Theorem}[section]
\newtheorem{proposition}[theorem]{Proposition}
\newtheorem{lemma}[theorem]{Lemma}
\newtheorem{corollary}[theorem]{Corollary}

\theoremstyle{definition}
\newtheorem{assumption}[theorem]{Assumption}

\theoremstyle{remark}
\newtheorem{remark}[theorem]{Remark}

\crefname{hypothesis}{Hypothesis}{Hypotheses}
\crefname{fact}{Fact}{Facts}
\crefname{assumption}{Assumption}{Assumptions}

\title{Diffusion Based Unpaired Data Learning for Inverse Problems}

\author{Chenglong Bao\thanks{Yau Mathematical Sciences Center, Tsinghua University, Beijing, China
  (\email{clbao@tsinghua.edu.cn}).}
\and Yiming Dang\thanks{Department of Mathematical Sciences, Tsinghua University, Beijing, China 
  (\email{dym22@mails.tsinghua.edu.cn}).}
\and Chenguang Duan\thanks{Institut f{\"u}r Geometrie und Praktische Mathematik, RWTH Aachen University, Aachen, Germany
  (\email{duan@igpm.rwth-aachen.de}).}
\and Yuling Jiao\thanks{School of Artificial Intelligence, Wuhan University, Wuhan, Hubei, China
(\email{yulingjiaomath@whu.edu.cn}).}
\and Defeng Sun\thanks{Department of Applied Mathematics, The Hong Kong Polytechnic University, Hong Kong, China
  (\email{defeng.sun@polyu.edu.hk}).}
}
\date{}

\hypersetup{
  pdftitle={Diffusion Based Unpaired Data Learning for Inverse Problems},
  pdfauthor={Chenglong Bao, Yiming Dang, Chenguang Duan, Yuling Jiao, Defeng Sun}
}

\begin{document}

\maketitle

\begin{abstract}
    Data is important in many deep learning-based inverse problem solvers. However, obtaining sufficient paired data in many scenarios remains highly challenging, while unpaired data is cheap. To maximize data utilization, this paper proposes LUD-DIF, a diffusion-based approach for solving inverse problems with unpaired data. Starting from the evidence lower bound (ELBO) of the joint distribution, we decouple it into two independent diffusion processes under the weak-coupling assumption. The method provides theoretical support from a variational inference perspective, derives the loss function, quantitatively analyzes the error bound introduced by the assumption, and offers a theorem-motivated heuristic for hyperparameter selection. Experimental results demonstrate that LUD-DIF achieves outstanding performance on multiple image inverse problems, validating its effectiveness and generalization capability in unpaired inverse problem settings.
\end{abstract}



\section{Introduction}\label{intro}
Inverse problems are widely encountered across numerous disciplines including computer vision~\cite{ribes2008linear}, geophysical exploration~\cite{wiggins1972general}, and biomedical imaging~\cite{bertero2006inverse}. These problems share a classical mathematical framework:
\[y = \mathcal{A}(x) + \Xi,\]
where $\mathcal{A}$ denotes the forward observation operator and $\Xi$ represents the observation noise. The goal of solving inverse problems is to robustly recover the unknown state $x$ from the observed data $y$. However, in real-world scenarios, the noise distribution is often unknown, and the observation operators typically exhibit highly nonlinear or low-rank characteristics, leading to the typical ill-posedness of inverse problems~\cite{kabanikhin2011inverse}. Traditional methods typically rely on certain prior assumptions (such as sparsity, smoothness, or low-rank structures)~\cite{tropp2010computational, spantini2015optimal} or employ maximum a posteriori (MAP) estimation and Bayesian inference under the assumption of Gaussian noise~\cite{stuart2010inverse}.

In recent years, deep learning techniques have brought a paradigm shift to solving inverse problems by training deep neural networks to learn complex mappings from observations to the underlying states directly from data~\cite{kamyab2022deep}.  
These learning-based approaches have surpassed traditional model-based regularization methods in many fields. They are capable of implicitly capturing high-dimensional and nonlinear data distributions and have achieved remarkable success in classical inverse problems such as image processing~\cite{ongie2020deep}, seismic wave inversion~\cite{adler2021deep}, and X-ray imaging~\cite{genzel2022near}.

Nevertheless, the exceptional performance of such end-to-end learning methods heavily relies on a prerequisite that is difficult to satisfy in practical applications: the availability of a large number of paired samples $\{(x_i, y_i)\}\sim p(x,y)$ for supervised training. In some real-world scenarios, such as medical or astronomical imaging, acquiring paired data is prohibitively expensive or entirely infeasible. Researchers are often restricted to accessing only independent ground-truth data $\{x_i\}_{i=1}^N \sim p(x)$ and observation data $\{y_j\}_{j=1}^M \sim p(y)$. This limitation has emerged as the primary bottleneck restricting the widespread application of deep learning in complex inverse problems. To alleviate this issue, some studies have attempted to perform self-supervised or unsupervised learning relying exclusively on the observation data $\{y_j\}$~\cite{lehtinen2018noise2noise, krull2019noise2void}; however, these approaches typically depend on strong heuristic priors. Another line of research utilizes only the ground-truth data $\{x_i\}$ to pre-train deep generative priors $p(x)$, which are subsequently combined with posterior sampling for problem resolution~\cite{daras2024survey}. Yet, such methods generally require exact knowledge of both the forward observation operator and the noise model.

Unpaired learning attempts to simultaneously exploit information from the marginal distributions $p(x)$ and $p(y)$ but without paired samples. From a probabilistic perspective, the essence of unpaired learning lies in inferring and modeling the joint distribution $p(x,y)$ under the constraint of these marginal distributions. Researchers often formulate unpaired learning as a generative problem, i.e., the generation of the
corresponding $y$ from a given $x$. Following this line, Generative Adversarial Networks (GANs) have been extensively applied due to their powerful cross-domain generation capabilities~\cite{li2019asymmetric, wei2020semi, arhire2025learned}. Nevertheless, these methods suffer from mode collapse and training instability, and they lack a rigorous probabilistic explanation. Variational Autoencoder (VAE)-based methods offer enhanced interpretability by introducing latent variables and optimizing the Evidence Lower Bound (ELBO)~\cite{zheng_learn_2023}; however, constrained by strong assumptions regarding the latent space distribution, their generation quality is often significantly compromised. Recent diffusion-based methods have begun to address conditional generation from unpaired data. OTCS constructs an optimal-transport coupling and trains a conditional score model using unpaired or partially paired data~\cite{gu_optimal_2023}. Diffusion Distribution Matching learns an unknown forward operator from unpaired clean and degraded images through conditional flow matching~\cite{meanti_unsupervised_2025}, whereas RealDGen synthesizes paired super-resolution data from unpaired HR and LR images using a content--degradation-decoupled diffusion model~\cite{peng_towards_2025}. Schrödinger-bridge methods provide another coupling-based approach to unpaired translation, with recent bridge-flow formulations reducing repeated diffusion-model training~\cite{liu_generalized_2024,gushchin_adversarial_2024,kim2023unpaired,de_bortoli_schrodinger_2024}. Relatedly, DRDD interprets Gaussian perturbation as a domain-harmonization mechanism for data-efficient image translation~\cite{lin_decoupled_2026}.

Despite this progress, a principled framework for deriving a likelihood-based conditional diffusion objective from fully unpaired marginals and quantifying the bias induced by the surrogate coupling remains underexplored. To address this gap, we propose LUD-DIF (\textbf{L}earning from \textbf{U}npaired \textbf{D}ata via \textbf{Dif}fusion Model) for inverse problems with unpaired data. LUD-DIF constructs proxy pairs through diffusion alignment and derives a trainable dual-path conditional diffusion objective from a marginal--conditional decomposition of the joint ELBO. This formulation relies solely on unpaired marginal samples and explicitly quantifies the bias induced by the weak-coupling approximation. The main contributions are summarized as follows:

\begin{itemize}
    \item From the perspective of variational inference, we decouple the ELBO of the joint distribution under a reasonable data prior assumption. This provides a feasible training algorithm for solving inverse problems with unpaired data using diffusion models.
    \item We quantitatively analyze the theoretical error bound introduced by our assumptions. Based on this analysis, we propose a theorem-motivated heuristic for selecting the hyperparameters $(S,V)$, establishing a theoretically grounded, practical strategy for model tuning.
    \item We validate the proposed method on natural image denoising and super-resolution datasets. Experimental results demonstrate that LUD-DIF achieves outstanding performance across these tasks.
\end{itemize}

\section{Preliminaries}\label{preli}

\par Diffusion models~\cite{ho_denoising_2020,song2020score} learn to approximate an unknown data distribution $p_{U_0}$ by gradually adding noise to the data and then learning to reverse this process. We briefly introduce the main components of diffusion models.

\paragraph{Forward process}
Let $T\in\mathbb{N}_{+}$. The forward process adds noise via a fixed Markovian chain:
\begin{align}\label{eq:forward:X}
p_{U_{1:T}|U_{0}}(u_{1:T}|u_{0}) \coloneq \prod_{t=1}^{T}p_{U_{t}|U_{t-1}}(u_{t}|u_{t-1}), 
p_{U_{t}|U_{t-1}}(\cdot|u_{t-1}) \coloneq \mathcal{N}(\sqrt{\alpha_{t}}u_{t-1},(1-\alpha_{t})I_{d}),
\end{align}
where $\{\alpha_t\}_{t=1}^T \subset (0, 1)$ is a prescribed variance schedule~\cite{ho_denoising_2020}. Let $\bar{\alpha}_{0}\coloneqq 1$ and $\bar{\alpha}_t \coloneq \prod_{i=0}^{t}\alpha_{i}$. Iterating this recursion yields the closed-form conditional marginal distribution for $U_{t}$ given $U_{0}=u_{0}$:
\begin{equation}\label{eq:forward:X:t}
p_{U_{t}|U_{0}}(\cdot|u_{0}) = \mathcal{N}(\sqrt{\bar{\alpha}_{t}}u_{0},(1-\bar{\alpha}_{t})I_{d}).
\end{equation}
As $\bar{\alpha}_T \to 0$, the final distribution $p_{U_{T}|U_{0}}$ converges to a standard Gaussian $\mathcal{N}(0, I_d)$.

\paragraph{Reverse process} 
To construct a generative model, the reverse process is parameterized as a Markovian chain with learned transition kernels starting from $p_{U_{T}}(u_{T}) \coloneq \mathcal{N}(0, I_d)$:
\begin{equation}\label{eq:reverse:X}
p_{U_{t-1}|U_{t}}^{\theta}(\cdot|u_{t}) \coloneq \mathcal{N}(\mu_{t}^{\theta}(u_{t}),\sigma_{t}^{2}I_{d}),
\end{equation}
where $\mu_t^{\theta}:\mathbb{R}^{d}\to\mathbb{R}^{d}$ is a neural network and $\sigma_{t}^{2}$ is a prescribed variance schedule. 

\paragraph{ELBO formulation}
Diffusion models are trained by optimizing the Evidence Lower Bound (ELBO) of the log-likelihood. Through Markovian properties, the ELBO can be rewritten as a tractable decomposition~\cite[Appendix A]{ho_denoising_2020}:
\begin{equation}\label{eq:elbo:diffusion}
\begin{aligned}
\mathrm{ELBO}(\theta;u_{0}) 
&= -D_{\mathrm{KL}}(p_{U_{T}|U_{0}}(\cdot|u_{0})\|p_{U_{T}}) + \mathbb{E}\big[\log p_{U_{0}|U_{1}}^{\theta}(u_0|U_{1})\big|U_{0}=u_{0}\big] \\
&\quad - \sum_{t=2}^{T}\mathbb{E}\big[D_{\mathrm{KL}}\big(p_{U_{t-1}|U_{t},U_{0}}(\cdot|U_{t},u_{0})\| p_{U_{t-1}|U_{t}}^{\theta}(\cdot|U_{t})\big)\big|U_{0}=u_{0}\big].
\end{aligned}
\end{equation}
The first term is constant due to the construction of the forward process, and the second term is the reconstruction term.
Moreover, we have $p_{U_{t-1}|U_{t},U_{0}}(\cdot|u_{t},u_{0})=\mathcal{N}(\tilde{\mu}_{t}(u_{t},u_0),\tilde{\beta}_{t}I_{d})$, where
\begin{equation}\label{eq:mu:beta}
\tilde{\mu}_{t}(u_{t},u_{0}) = \frac{\sqrt{\bar{\alpha}_{t-1}}\beta_{t}}{1-\bar{\alpha}_t}u_0 + \frac{\sqrt{\alpha_t}(1 - \bar{\alpha}_{t-1})}{1 - \bar{\alpha}_t}u_t, \quad  \tilde{\beta}_t = \frac{(1 - \bar{\alpha}_{t-1})\beta_t}{1 - \bar{\alpha}_t},
\end{equation}
where $\beta_t \coloneq 1-\alpha_t$. Substituting~\eqref{eq:mu:beta} into~\eqref{eq:elbo:diffusion} yields the mean-squared error loss function:
\begin{equation*}
\mathcal{L}(\theta) = \mathbb{E}\left[\sum_{t=1}^{T}\frac{1}{2\sigma_{t}^{2}}\| \mu_{t}^{\theta}(U_{t})-\tilde{\mu}_{t}(U_{t},U_{0})\|_{2}^{2}\right],
\end{equation*}
where $\tilde{\mu}_{1}(u_{1},u_{0})\coloneq u_{0}$.

\section{The LUD-DIF method}\label{method}
\subsection{Problem Formulation}\label{prob reduction}

Consider a general inverse problem modeled as
\begin{equation}
    \label{eq:inverse_problem_model}
    Y_{0} = \mathcal{A}(X_{0}) + \Xi,
\end{equation}
where $\mathcal{A}$ is the forward operator and $\Xi$ denotes the noise. In the unpaired data setting, we are given samples from the marginal distributions
\[
X_{0}^{1}, \ldots, X_{0}^{N} \overset{\mathrm{i.i.d.}}{\sim} p_{X_{0}}, 
\qquad
Y_{0}^{1}, \ldots, Y_{0}^{N} \overset{\mathrm{i.i.d.}}{\sim} p_{Y_{0}},
\]
and the goal is to recover the joint distribution $p_{X_0, Y_0}$. If the noise distribution were known, one could directly characterize the conditional distribution $p_{Y_0 \mid X_0}$ via the forward model \eqref{eq:inverse_problem_model}, and thereby recover the joint distribution. However, in practice, the noise distribution is typically unknown and may be highly complex, which poses the main challenge in estimating $p_{X_0, Y_0}$.

In this work, we model the noise conditionally as $\Xi \sim p_{\Xi \mid \mathcal{A}(X_{0})}$. By treating $\mathcal{A}(X_0)$ as an effective signal variable and relabeling it as $X_0$, the forward model can be reduced to
\begin{align}
\label{eq:noise modeling}
    Y_{0} = X_{0} + \Xi, 
\qquad 
\Xi \sim p_{\Xi \mid X_{0}}.
\end{align}
This reformulation reduces the problem to estimating the conditional noise distribution, which we refer to as the \emph{noise modeling problem}. 

\begin{remark}
If the forward operator is unknown, it can be replaced by an approximate unbiased operator $\widehat{\mathcal{A}}$, i.e., $\mathbb{E}[\mathcal{A}(X_{0})]=\mathbb{E}[\widehat{\mathcal{A}}(X_{0})]$, that incorporates prior information. In this case, the model can be written as
\[
Y_0=\widehat{\mathcal A}(X_0)+\widetilde \Xi,
\qquad
\widetilde \Xi
=
\Xi+\mathcal A(X_0)-\widehat{\mathcal A}(X_0).
\]
The residual term is absorbed into the effective noise $\widetilde \Xi$, and the model is again reduced to the form in~\eqref{eq:noise modeling} with $\mathbb{E}[\widetilde{\Xi}]=\mathbb{E}[\Xi]$.
\end{remark}

\par Accordingly, in the following, we focus on this \emph{noise modeling problem}.

{\bf Data preprocessing.}  Let $\sigma_{X}^2 = \frac{1}{d}\mathbb{E}[\|X_0 - \mu_X\|_2^2]$ and $\sigma_{Y}^2 = \frac{1}{d}\mathbb{E}[\|Y_0 - \mu_Y\|_2^2]$, where $\mu_X=\mathbb{E}[X_0]$ and $\mu_Y=\mathbb{E}[Y_0]$. We normalize the data by
\begin{equation}\label{eq:data_prep}
\tilde{X}_0 = \frac{X_0}{\sigma_X},\quad \tilde{Y}_0 = \frac{Y_0}{\sigma_Y}.
\end{equation}
This normalization aligns the scales of the two marginal distributions, making their amplitudes comparable when constructing a joint diffusion process for $X_0$ and $Y_0$. 
For notational simplicity, we assume that this preprocessing step has been applied and, in the remainder of the paper, use $X_0$ and $Y_0$ to denote the normalized variables $\tilde{X}_0$ and $\tilde{Y}_0$.

\subsection{Learning joint distribution via diffusion models}

Let $Z_{0} \coloneq (X_{0}, Y_{0}) \sim p_{X_{0}, Y_{0}}$ denote the true but unobserved joint data distribution. We construct a diffusion process on the joint variable $Z_0$ by specifying an appropriate reverse process.

\paragraph{Forward process}
Treating the joint variable $Z_t \coloneq (X_t, Y_t)$ as the counterpart of $U_t$ in~\eqref{eq:forward:X}, we define the joint forward diffusion process as
\begin{equation}\label{eq:forward:joint}
X_{t} \coloneq \sqrt{\alpha_t} X_{t-1} + \sqrt{1-\alpha_t}\,\varepsilon_{t-1}^{X},
\quad 
Y_{t} \coloneq \sqrt{\alpha_t} Y_{t-1} + \sqrt{1-\alpha_t}\,\varepsilon_{t-1}^{Y},
\end{equation}
where $\varepsilon_{t-1}^{X}, \varepsilon_{t-1}^{Y} \sim \mathcal{N}(0, I_d)$ are independent standard Gaussian noise variables that are independent of $X_{t-1}$ and $Y_{t-1}$. We establish the following factorization result that decouples across $X_0$ and $Y_0$.
\begin{proposition}\label{prop:factorization} Let $Z_t=(X_t,Y_t)$ follow the forward process~\eqref{eq:forward:joint}. Then we have
\begin{subequations}
\begin{align}
p_{Z_{t}|Z_{0}}(z_{t}|z_{0}) &= p_{X_{t}|X_{0}}(x_{t}|x_{0})p_{Y_{t}|Y_{0}}(y_t|y_0), \label{eq:prop-forward} \\
p_{Z_{t-1}|Z_{t},Z_{0}}(z_{t-1}|z_{t},z_{0}) &= p_{X_{t-1}|X_{t},X_{0}}(x_{t-1}|x_{t},x_{0})p_{Y_{t-1}|Y_{t},Y_{0}}(y_{t-1}|y_{t},y_{0}), \label{eq:prop-posterior}
\end{align}
\end{subequations}
where $z_{t}\coloneq(x_{t},y_{t})$.
\end{proposition}
\begin{proof}
    See Appendix~\ref{proof:factorization}.
\end{proof}

\paragraph{Reverse process}
Since paired samples $(X_0,Y_0)\sim p_{X_0,Y_0}$ are unavailable in the unpaired setting, it is difficult to directly adopt the reverse process parameterization~\eqref{eq:reverse:X} used in classical diffusion models. 
To construct a tractable variational reverse process, we first note that, 
for the $X$-first autoregressive ordering, the exact reverse transition kernel 
admits the chain-rule factorization
\begin{equation*}
p_{Z_{t-1}\mid Z_t}(z_{t-1}\mid z_t)
=
p_{X_{t-1}\mid X_t,Y_t}(x_{t-1}\mid x_t,y_t)\,
p_{Y_{t-1}\mid X_{t-1},X_t,Y_t}(y_{t-1}\mid x_{t-1},x_t,y_t).
\end{equation*}
This identity only motivates the autoregressive ordering. In the variational 
model, we do not attempt to represent the two exact conditionals above. 
Instead, we restrict the reverse model to a tractable factored Gaussian family:
\begin{equation}
\label{eq:prop-reverse}
\begin{aligned}
p_{Z_{t-1}\mid Z_t}^{\theta}(z_{t-1}\mid z_t)
&\coloneqq
p_{X_{t-1}\mid X_t}^{\theta}(x_{t-1}\mid x_t)\,
p_{Y_{t-1}\mid X_{t-1},Y_t}^{\theta}(y_{t-1}\mid x_{t-1},y_t) \\
&=
\mathcal{N}\bigl(x_{t-1};\mu_{X,t}^{\theta}(x_t),\sigma_t^2 I_d\bigr)
\,
\mathcal{N}\bigl(y_{t-1};\mu_{Y\mid X,t}^{\theta}(y_t, x_{t-1}),\sigma_t^2 I_d\bigr).
\end{aligned}
\end{equation}
The omission of $y_t$ in the first factor and $x_t$ in the second factor 
should be understood as a restriction of the variational family, rather than 
as a conditional-independence assumption on the true reverse process. 
Intuitively, the $X$-reverse kernel is used to model the clean structural 
marginal, for which $x_t$ is the most direct noisy observation, while the 
$Y$-reverse kernel is conditioned on the less diffused structural state 
$x_{t-1}$ together with the noisy observation $y_t$. This design keeps the 
conditional diffusion model trainable from the weak-coupling samples introduced 
below. Since both factors in \eqref{eq:prop-reverse} are normalized conditional 
Gaussian densities, their product defines a valid conditional density over 
$(x_{t-1},y_{t-1})$ given $(x_t,y_t)$.

Based on the above parameterization, we show that the proposed reverse process admits a corresponding decomposition of the ELBO.

\paragraph{ELBO formulation}
Analogously to~\eqref{eq:elbo:diffusion}, we identify $Z_t$ with $U_t$ and define the objective function as the negative-ELBO loss function:
\begin{equation}\label{eq:objective:joint}
\mathcal{L}_{\mathrm{joint}}(\theta) \coloneq -\mathbb{E}_{Z_{0}}\left[\mathrm{ELBO}_{Z}(\theta; Z_{0})\right],
\end{equation}
where $\mathrm{ELBO}_Z$ is given by
\begin{equation}\label{eq:ELBO:joint}
\begin{aligned}
\mathrm{ELBO}_Z(\theta;z_{0}) \coloneq & -D_{\mathrm{KL}}(p_{Z_{T}|Z_{0}}(\cdot|z_{0})\|p_{Z_{T}})+\mathbb{E}\left[\log p_{Z_{0}|Z_{1}}^{\theta}(z_0|Z_{1})|Z_{0}=z_{0}\right] \\
&- \sum_{t=2}^{T}\mathbb{E}\left[D_{\mathrm{KL}}\big(p_{Z_{t-1}|Z_{t},Z_{0}}(\cdot|Z_{t},z_{0})\|p_{Z_{t-1}|Z_{t}}^{\theta}(\cdot|Z_{t})\big)|Z_{0}=z_{0}\right].
\end{aligned}
\end{equation}
Since paired samples $z_0=(x_0,y_0)\sim p_{X_0,Y_0}$ are unavailable, we decompose the joint ELBO under the variational family in~\eqref{eq:prop-reverse} into an $X$-marginal ELBO and a conditional $Y|X$ ELBO.

\begin{proposition}[Decoupling of joint ELBO]\label{proposition:objective:groundtruth}
Let the joint diffusion process $\{Z_t\}_{t=0}^T$ be defined by~\eqref{eq:forward:joint} and~\eqref{eq:prop-reverse}, and let $\mathrm{ELBO}_{Z}$ be given by~\eqref{eq:ELBO:joint}. Then the following decomposition holds:
\begin{equation*}
\mathrm{ELBO}_{Z}(\theta; z_{0})
= \mathrm{ELBO}_{X}(\theta; x_{0})
+ \mathrm{ELBO}_{Y \mid X}(\theta; z_{0})
+ C(z_{0}),
\end{equation*}
where $C(z_{0})$ is a constant independent of $\theta$. The marginal ELBO $\mathrm{ELBO}_X$ is
\begin{align}
\label{eq:marginal-ELBO}
\mathrm{ELBO}_{X}(\theta;x_{0}) 
&= \mathbb{E}\left[\log p_{X_{0}|X_{1}}^{\theta}(x_{0}|X_{1})|X_{0}=x_{0}\right] \nonumber\\
    &\quad - \sum_{t=2}^{T}\mathbb{E}\left[D_{\rm KL}(p_{X_{t-1}|X_{t},X_{0}}(\cdot|X_{t},x_{0}) \| p_{X_{t-1}|X_{t}}^{\theta}(\cdot|X_{t})) \Big| X_{0}=x_{0}\right].
\end{align}
The conditional ELBO $\mathrm{ELBO}_{Y|X}$ is
\begin{align}
\label{eq:conditional-ELBO}
\mathrm{ELBO}_{Y|X}(\theta;z_0) 
&= \mathbb E\bigl[\log p^\theta_{Y_0\mid X_0,Y_1}(y_0\mid x_0,Y_1)\mid Z_0=z_0\bigr] \nonumber\\
&\quad - \sum_{t=2}^{T}\mathbb{E}\Bigl[D_{\rm KL}\Bigl(p_{Y_{t-1}|Y_{t},Y_{0}}(\cdot|Y_{t},y_{0}) \;\Big\|\; p_{Y_{t-1}\mid X_{t-1},Y_{t}}^{\theta}(\cdot \mid X_{t-1},Y_{t})\Bigr) \;\Big|\; Z_{0}=z_{0}\Bigr].
\end{align}
\end{proposition}
\begin{proof}
    See Appendix~\ref{proof:objective:groundtruth}.
\end{proof}
Using the parameterizations in~\eqref{eq:prop-reverse}, we obtain, up to a constant,
\begin{align}
     & -\mathbb{E}[\mathrm{ELBO}_X(\theta;x_0)] = \mathbb{E}\sum_{t=1}^{T}\frac{1}{2\sigma_{t}^{2}}\|\mu_{X,t}^{\theta}(X_{t})-\tilde{\mu}_{t}(X_{t},X_{0})\|_{2}^{2}\eqqcolon L_{\mathrm{marg}}(\theta), \label{eq:Loss-marginal} \\
     & -\mathbb{E}[\mathrm{ELBO}_{Y|X}(\theta;z_0)] = \mathbb{E}\sum_{t=1}^{T}\frac{1}{2\sigma_{t}^{2}}\|\mu_{Y\mid X,t}^{\theta}(Y_{t}, X_{t-1})-\tilde{\mu}_{t}(Y_{t},Y_{0})\|_{2}^{2}\eqqcolon L_{\mathrm{cond}}(\theta). \label{eq:Loss-condition}
\end{align}
Collecting \(-\mathbb{E}[C(Z_0)]\) and all other terms independent of \(\theta\) into a constant \(C\), the overall loss function becomes
\begin{equation}
\label{eq:joint:objective}
    \mathcal{L}_{\mathrm{joint}}(\theta) 
\coloneq -\mathbb{E}\left[\mathrm{ELBO}_{Z}(\theta;Z_{0})\right] = L_{\mathrm{marg}}(\theta) + L_{\mathrm{cond}}(\theta) + C.
\end{equation}

This formulation serves as the ideal loss for learning the joint distribution. Notably, $L_{\mathrm{marg}}(\theta)$ can be estimated using samples from $p_{X_0}$, whereas $L_{\mathrm{cond}}(\theta)$ requires samples from the joint distribution of $(X_{t-1}, Y_t, Y_0)$, which depends on
\begin{equation*}
p_{X_{t-1},Y_{t},Y_{0}}(x_{t-1},y_{t},y_{0})
= \int p_{X_{t-1}|X_{0}}(x_{t-1}|x_{0})\, p_{X_{0}|Y_{0}}(x_{0}|y_{0})\, p_{Y_{t},Y_{0}}(y_{t},y_{0})\, \mathrm{d}x_{0}.
\end{equation*}
However, the available marginal distributions $p_{X_0}$ and $p_{Y_0}$ alone do not uniquely determine the joint distribution. As a result, the conditional distribution $p_{X_0 \mid Y_0}$ remains inaccessible in the unpaired setting. In many inverse problems, high-dimensional noisy signals are more easily obscured by additional Gaussian perturbations than are the underlying dominant structures. Motivated by this observation, we introduce the following weak-coupling assumption.

\begin{assumption}[Weak coupling]\label{assum:weak:coupling}
There exist time steps $0\leq V, S \leq T$ such that 
\begin{equation*}
p_{X_{0}|Y_{0}}(\cdot|y_{0}) = p_{X_{0}|Y_{0}}^{\rm weak}(\cdot|y_{0}), \quad y_{0}\in\mathbb{R}^{d},
\end{equation*}
where $p_{X_{0}|Y_{0}}^{\rm weak}$ and $p_{X_{0}|Y_{S}}^{\rm weak}$ are defined as
\begin{align}
    p_{X_{0}|Y_{0}}^{\rm weak}(\cdot|y_{0}) & \coloneqq \mathbb{E}_{\xi\sim\mathcal{N}(0,I_{d})}\left[p_{X_{0}|Y_{S}}^{\rm weak}(\cdot|\sqrt{\bar{\alpha}_{S}}y_{0}+\sqrt{1-\bar{\alpha}_{S}}\xi)\right], \label{eq:weak:coupling:denoising} \\
    p_{X_{0}|Y_{S}}^{\rm weak}(\cdot|w) & \coloneqq p_{X_{0}|X_{V}}(\cdot|w), \quad \forall w\in\mathbb{R}^{d}. \label{eq:weak:coupling:y2x}
\end{align}
\end{assumption}
\begin{remark}
\label{rem:gaussian_weak_coupling}
    For Gaussian white noise with variance $\sigma^2$, this coupling assumption is exact when we set $S=0$ and choose $V$ such that $\bar{\alpha}_V = \frac{\sigma_X^2}{\sigma_X^2+\sigma^2}$, thereby matching the signal-to-noise ratio of the measurement.
\end{remark}

\begin{remark} 
For the general case, the assumption relies on two approximations:
(i) {Alignment approximation}:
$p_{X_0|Y_S}(\cdot \mid w) \approx p_{X_0|X_V}(\cdot \mid w)$,
which assumes that the noisy representations of the two modalities become structurally indistinguishable at suitable noise levels;
(ii) {Perturbation approximation}:
$
p_{X_{0}|Y_{0}}(\cdot \mid y_{0}) \approx \mathbb{E}_{\xi \sim \mathcal{N}(0, I_{d})}
\left[p_{X_{0}|Y_{S}}\big(\cdot \mid \sqrt{\bar{\alpha}_{S}}\, y_{0} + \sqrt{1-\bar{\alpha}_{S}}\, \xi\big)\right],
$
which assumes that the perturbation $Y_S$ preserves sufficient information about $X_0$.
These two approximations induce a trade-off in the choice of the time steps $S$ and $V$. When $S$ is small, the perturbation approximation is nearly exact, whereas the alignment approximation is difficult to satisfy. As $S$ and $V$ increase, the alignment approximation improves, but the perturbation approximation deteriorates due to information loss. Therefore, the practical effectiveness of the weak-coupling assumption depends on selecting $S$ and $V$ to balance these competing effects.
\end{remark}

The weak-coupling distribution can be sampled using only marginal data:
\begin{enumerate}
\item Draw $(Y_{0},\xi)\sim p_{Y_{0}}\otimes\mathcal{N}(0,I_d)$;
\item Define $\bar{Y}_{S}=\sqrt{\bar{\alpha}_S}\,Y_0+\sqrt{1-\bar{\alpha}_S}\,\xi$;
\item Draw $X_{0}^{\mathrm{weak}}\sim p_{X_{0}\mid X_{V}}(\cdot\mid\bar{Y}_{S})$ via reverse process.
\end{enumerate}
By replacing the inaccessible conditional distribution $p_{X_0|Y_0}$ with the proxy conditional distribution $p_{X_0|Y_0}^{\rm weak}$ in~\eqref{eq:Loss-condition}, we define the weak conditional loss as 
\begin{equation}\label{eq:conditional-loss-weak}
L_{\mathrm{cond}}^{\mathrm{weak}}(\theta) \coloneq  \mathbb{E}\Bigg[\sum_{t=1}^{T}\frac{1}{2\sigma_{t}^{2}}\|\mu_{Y\mid X,t}^{\theta}(Y_t,X_{t-1}^{\rm weak})-\tilde{\mu}_{t}(Y_{t},Y_{0})\|_{2}^{2}\Bigg],
\end{equation}
where 
\begin{equation*}
X_{t-1}^{\mathrm{weak}}\coloneqq\sqrt{\bar{\alpha}_{t-1}}X_{0}^{\rm weak}+\sqrt{1-\bar{\alpha}_{t-1}}\varepsilon^{X}, \quad Y_{t}\coloneqq\sqrt{\bar{\alpha}_{t}}Y_{0}+\sqrt{1-\bar{\alpha}_{t}}\varepsilon^{Y},
\end{equation*}
with $\varepsilon^{X},\varepsilon^{Y}\sim\mathcal{N}(0,I_{d})$ independent of each other and independent of $X_{0}^{\mathrm{weak}}$ and $Y_{0}$. Therefore, $X_{0}^{\mathrm{weak}}$ is conditionally independent of $Y_{t}$ given $Y_{0}$. This modification avoids the need for paired data and makes the conditional loss computable. The next theorem shows the equivalence between $L_{\mathrm{cond}}^{\mathrm{weak}}(\theta)$ in~\eqref{eq:conditional-loss-weak} and $L_{\mathrm{cond}}(\theta)$ in~\eqref{eq:Loss-condition} under Assumption~\ref{assum:weak:coupling}. We define
\begin{equation}\label{eq:objective:weak}
\mathcal{L}_{\rm weak}(\theta) \coloneq L_{\text{marg}}(\theta) + L_{\text{cond}}^{\text{weak}}(\theta) + C,
\end{equation}
where $L_{\mathrm{marg}}(\theta)$ is defined in \eqref{eq:Loss-marginal} and $L_{\mathrm{cond}}^{\mathrm{weak}}(\theta)$ is defined in \eqref{eq:conditional-loss-weak}.

\begin{theorem}\label{theorem:equivalence}
Suppose that Assumption~\ref{assum:weak:coupling} holds. Let $\mathcal{L}_{\mathrm{joint}}(\theta)$ and $\mathcal{L}_{\rm weak}(\theta)$ be defined as~\eqref{eq:joint:objective} and~\eqref{eq:objective:weak}, respectively. Then we have
\begin{equation*}
L_{\mathrm{cond}}(\theta) = L_{\mathrm{cond}}^{\mathrm{weak}}(\theta).
\end{equation*}
Thus, $\mathcal{L}_{\mathrm{joint}}(\theta) = \mathcal{L}_{\rm weak}(\theta)$.
\end{theorem}

\begin{proof}
    See Appendix~\ref{proof:equivalence}.
\end{proof}
Theorem~\ref{theorem:equivalence} provides a consistency result: when the weak-coupling construction matches the true posterior coupling, the computable weak objective coincides with the ideal paired-data objective. When the assumption is relaxed, Section~\ref{sec:analysis} quantifies the resulting bias. Algorithm~\ref{train} summarizes the training procedure. Since the objective is decoupled into marginal and conditional terms, the two diffusion models can be trained separately, improving the training efficiency of the conditional model.

\begin{algorithm}[H]
\caption{LUD-DIF Algorithm: Training Process}
\label{train}
\begin{algorithmic}[1]
\STATE Initialize parameters $S, V, \alpha_t$ of the model.
\FOR{each training iteration}
\STATE Sample $X_0 \sim p_{X_0}$, $Y_0 \sim p_{Y_0}$ and $t \sim \text{Uniform}\{1, \ldots, T\}$.
\STATE Generate $X_t \sim p_{X_t|X_0}(\cdot|X_0)$ and $Y_t \sim p_{Y_t|Y_0}(\cdot|Y_0)$ via the forward process \eqref{eq:forward:joint}.
\STATE Generate proxy sample $X_0^{\rm weak} \sim p_{X_0|Y_0}^{\rm weak}(\cdot|Y_0)$ as defined in \eqref{eq:weak:coupling:denoising} and \eqref{eq:weak:coupling:y2x}.
\STATE Construct $X_{t-1} = \sqrt{\bar{\alpha}_{t-1}}sg(X_{0}^{\rm weak})+\sqrt{1-\bar{\alpha}_{t-1}}\xi$ where $\xi \sim \mathcal{N}(0, I_d)$.
\STATE Compute $\mathcal{L}_X = \frac{1}{2\sigma_t^2}\|\mu^{\theta}_{X,t}(X_t) - \tilde{\mu}_t(X_t,X_0)\|^2$.
\STATE Compute $\mathcal{L}_{Y|X} = \frac{1}{2\sigma_t^2}\|\mu^{\theta}_{Y\mid X,t}(Y_t, X_{t-1}) - \tilde{\mu}_t(Y_t,Y_0)\|^2$.
\STATE Update $\theta$ by minimizing $\mathcal{L}_{\rm weak} = \mathcal{L}_X + \mathcal{L}_{Y|X}$.
\ENDFOR
\end{algorithmic}
\end{algorithm}

\section{The Bias Analysis without Assumption~\ref{assum:weak:coupling}}\label{sec:analysis}

Theorem~\ref{theorem:equivalence} establishes the equivalence between $\mathcal{L}_{\mathrm{joint}}$ and $\mathcal{L}_{\mathrm{weak}}$ under Assumption~\ref{assum:weak:coupling}. Here, we quantify their discrepancy when this assumption is relaxed:
\begin{equation}
\Delta \coloneq |\mathcal{L}_{\mathrm{joint}}-\mathcal{L}_{\mathrm{weak}}|.
\end{equation}
Throughout this section, $p_{X_0|X_V}$ denotes the exact conditional distribution, so learning and sampling errors are excluded. Following the normalization in~\eqref{eq:data_prep}, we consider
\begin{equation}\label{eq:normalization}
Y_0=\frac{\sigma_X}{\sigma_Y}X_0+\frac{1}{\sigma_Y}\Xi,
\end{equation}
where $\Xi$ is assumed to be independent of $X_0$. A concise extension to signal-dependent conditional noise is given in Appendix~\ref{supp:sec:conditional-analysis}. The resulting bounds also guide the selection of the weak-coupling time steps $S$ and $V$. We first state the required technical assumptions.

\begin{assumption}[Polynomial growth of networks]\label{assum:network}
The expectation neural network $\mu^{\theta}_{Y\mid X,t}:\mathbb{R}^{d}\times\mathbb{R}^{d}\to\mathbb{R}^{d}$ defined in~\eqref{eq:prop-reverse} is bounded by a polynomial uniformly over the parameter set $\Theta$, i.e., there exist $C > 0$ and $m \geq 1$ such that for any $\theta\in\Theta$ and $y_{t},x\in\mathbb{R}^d$,
\begin{equation*}
\|\mu_{Y\mid X,t}^{\theta}(y_{t},x)\|_{2}^{2} \leq C(1+\|y_{t}\|_{2}^{m}+\|x\|_{2}^{m}).
\end{equation*}
\end{assumption}

\par The polynomial growth of the neural network in Assumption~\ref{assum:network} can be ensured in practice by truncation or weight clipping. Note that the expectation $\tilde{\mu}_t(y_t, y_0)$ defined in~\eqref{eq:mu:beta} is a linear combination of $y_t$ and $y_0$; thus, under Assumption~\ref{assum:network}, there exists a constant $K > 0$ such that
\begin{equation*}
\|\mu_{Y\mid X,t}^{\theta}(y_t, x) - \tilde{\mu}_{t}(y_t, y_0)\|_{2}^{2} \leq K(1 + \|y_t\|_{2}^{\max\{m,2\}} + \|x\|_{2}^m + \|y_0\|_{2}^2).
\end{equation*}
Here $K$ is a constant depending only on $C$ and the variance schedule of the diffusion model.

\par Before proceeding with the data distributions, we note the necessary variable assumptions supporting the scaling properties.

\begin{assumption}[Data and noise distribution]\label{assum:data}
The normalized clean data $X_0$ admits a probability density 
$p_{X_0}$ with respect to the Lebesgue measure on $\mathbb{R}^d$ 
and is supported on a compact set 
$\mathcal{X} \subset \mathbb{R}^d$. 
The noise $\Xi$ is independent of $X_0$ and is centered, i.e.,
\[
\mathbb{E}[\Xi] = \mathbf{0}.
\]
Moreover, $\Xi$ admits a probability density $p_{\Xi}$ satisfying
\[
\|p_{\Xi}\|_{L^\infty(\mathbb{R}^d)} \leq M_{\Xi},
\]
and has a bounded $2\max\{m,2\}$-th moment:
\[
\mathbb{E}\!\left[
    \|\Xi\|_2^{2\max\{m,2\}}
\right]
\leq M_{\Xi,m},
\]
for some finite constants $M_{\Xi}, M_{\Xi,m} > 0$.
\end{assumption}

\par Since $Y_0$ is a linear combination of the bounded $X_0$ and the noise $\Xi$, it naturally follows that the noisy observation $Y_0$ has finite $2\max(m, 2)$-th order moments.

\par The Total Variation (TV) distance between two probability density functions $p_1$ and $p_2$ is defined as $\|p_1 - p_2\|_{\mathrm{TV}} \coloneq \frac{1}{2} \int |p_1(x) - p_2(x)|\,\mathrm{d}x$. The following proposition shows that the error of the weakly coupled loss function can be controlled by the expected TV distance between the true posterior $p_{X_0|Y_0}(\cdot|Y_0)$ and the weakly coupled posterior $p_{X_0|Y_0}^{\mathrm{weak}}(\cdot|Y_0)$.

\begin{proposition}\label{proposition:delta}
Suppose Assumptions~\ref{assum:network} and~\ref{assum:data} are fulfilled. Then there exists a sequence of positive finite constants $\{M_t\}_{t=1}^T$ such that the error function is bounded by
\begin{equation*}
\Delta \leq \sum_{t=1}^{T}\frac{M_t}{\sigma_{t}^{2}}\mathbb{E}_{Y_0}^{\frac{1}{2}}\left[\|p_{X_{0}|Y_{0}}(\cdot|Y_{0})-p_{X_{0}|Y_{0}}^{\mathrm{weak}}(\cdot|Y_{0})\|_{\mathrm{TV}}\right],
\end{equation*}
where $M_{t}$ is a constant depending on $C$, $R_{\mathcal{X}} \coloneqq \sup_{x\in\mathcal{X}}\|x\|_2$, dimension $d$, moment bound $M_{\Xi,m}$, scales $\sigma_X, \sigma_Y$, and the variance schedule.
\end{proposition}

\begin{proof}
See Appendix~\ref{proof:delta}.
\end{proof}

\par From Proposition~\ref{proposition:delta}, in order to estimate the error $\Delta$ of the weakly coupled loss function, it is sufficient to bound the expected TV distance between the weakly coupled posterior and the true posterior:
\begin{equation}
\Delta_p \coloneq \mathbb{E}_{Y_0} \left[ \| p_{X_{0}|Y_{0}}^{\mathrm{weak}}(\cdot|Y_0) - p_{X_{0}|Y_{0}}(\cdot|Y_0) \|_{\mathrm{TV}} \right].
\end{equation}
Recall that $p_{X_0|Y_0}^{\text{weak}}(\cdot|Y_0) = \mathbb{E}_{Y_S|Y_0}[p_{X_0|X_V}(\cdot|Y_S)]$. By Jensen's inequality, the expected TV distance $\Delta_{p}$ between the weakly coupled posterior and the true posterior can be decomposed as
\begin{align}
\Delta_p 
&\leq \mathbb{E}_{Y_0}\mathbb{E}_{Y_S|Y_0} \left[ \|p_{X_0|X_V}(\cdot|Y_S) - p_{X_0|Y_0}(\cdot|Y_0)\|_{\mathrm{TV}} \right] \nonumber \\
&\leq \underbrace{\mathbb{E}_{Y_S} \left[ \|p_{X_0|X_V}(\cdot|Y_S) - p_{X_0|Y_S}(\cdot|Y_S)\|_{\mathrm{TV}} \right]}_{\text{alignment error}}+\underbrace{\mathbb{E}_{Y_0, Y_S} \left[ \|p_{X_0|Y_S}(\cdot|Y_S) - p_{X_0|Y_0}(\cdot|Y_0)\|_{\mathrm{TV}} \right]}_{\text{perturbation error}}. \label{eq:error:decomp}
\end{align}
The first summand in~\eqref{eq:error:decomp} represents the \textbf{alignment error}, arising from the structural discrepancy between the noisy signal $X_V$ and the noisy measurement $Y_S$. This corresponds to the error introduced by the alignment approximation in~\eqref{eq:weak:coupling:y2x}. The second summand represents the \textbf{perturbation error}, which stems from the information loss in $Y_S$ due to the perturbation to $Y_0$. 

\par To establish rigorous bounds for these errors, we introduce the stability assumption on the posterior distribution with respect to the change of observation.

\begin{assumption}[Posterior stability]\label{assum:compactness}
There exists a constant $L_{\mathrm{post}}>0$ such that
\[\|p_{X_0|Y_0}(\cdot|y_{0}^{1}) - p_{X_0|Y_0}(\cdot|y_{0}^{2})\|_{\mathrm{TV}} \leq L_{\mathrm{post}} \|y_{0}^{1}-y_{0}^{2}\|_2, \quad \forall \, y_{0}^{1}, y_{0}^{2} \in \mathbb{R}^d.\]
\end{assumption}

\par The stability of the posterior distribution with respect to perturbations of the observations has been extensively studied in the context of Bayesian inverse problems~\cite{stuart2010inverse}. We adopt it here as a mild regularity assumption. Appendix~\ref{supp:sec:conditional-analysis} records a sufficient condition: under Assumption~\ref{assum:data}, a positive noise density whose log-density has a bounded Hessian induces a TV-Lipschitz posterior map. This condition covers Gaussian noise and a broad class of non-Gaussian cases.

\par The following theorem provides explicit upper bounds for the perturbation error and the alignment error in terms of the diffusion time steps $S, V$ and the noise $\Xi$ under the stated assumptions.

\begin{theorem}[Error bound for the weakly coupled loss function]\label{thm:error_bound}
Suppose Assumptions~\ref{assum:network},~\ref{assum:data}, and~\ref{assum:compactness} are fulfilled. Let $\sigma_{\Xi}^{2}\coloneq \frac{1}{d}\mathbb{E}[\|\Xi\|_{2}^2]$ be the scalar variance of the noise $\Xi$. Then 
\begin{equation*}
\Delta_{p} \leq \underbrace{L_{\mathrm{post}}\sqrt{2d\frac{1-\bar{\alpha}_S}{\bar{\alpha}_S}}}_{\text{perturbation error}}+\underbrace{\zeta_{S}\frac{|\sigma_Y \sqrt{\bar{\alpha}_V} - \sigma_X \sqrt{\bar{\alpha}_S}|}{\sqrt{1-\bar{\alpha}_V}}+\eta_{S}\sqrt{D_{\mathrm{KL}}(p_{\Xi} \,\|\, \mathcal{N}(0, \sigma_{\Xi}^2 I_d))}}_{\text{alignment error}},
\end{equation*}
where the constants $\zeta_{S}$ and $\eta_{S}$ are defined as 
\begin{equation*}
\zeta_{S} \coloneqq \frac{R_{\mathcal{X}}}{\sigma_Y}+\sqrt{\frac{8d}{\sigma_Y^2 - \bar{\alpha}_S \sigma_X^2}}, \quad \eta_{S} \coloneqq \sqrt{\frac{2\bar{\alpha}_S \sigma_{\Xi}^2}{\sigma_Y^2 - \bar{\alpha}_S \sigma_X^2}}.
\end{equation*}
Moreover, the error function $\Delta$ satisfies
\begin{equation*}
\Delta \leq \sum_{t=1}^{T}\frac{M_t}{\sigma_{t}^{2}}\Biggl(L_{\mathrm{post}}\sqrt{2d\frac{1-\bar{\alpha}_S}{\bar{\alpha}_S}}+\zeta_{S}\frac{|\sigma_Y \sqrt{\bar{\alpha}_V} - \sigma_X \sqrt{\bar{\alpha}_S}|}{\sqrt{1-\bar{\alpha}_V}}+\eta_{S}\sqrt{D_{\mathrm{KL}}(p_{\Xi} \,\|\, \mathcal{N}(0, \sigma_{\Xi}^2 I_d))}\Biggr)^{\frac{1}{2}}.
\end{equation*}
\end{theorem}
\begin{proof}
    See Appendix \ref{proof:error bound}.
\end{proof}

\par We define a conditional distribution in~\eqref{eq:weak:coupling:denoising}, which induces a conditional distribution of $Y_{0}$ given $X_{0}^{\mathrm{weak}}$:
\begin{equation}\label{eq:condition:yx}
p_{Y_{0}|X_{0}}^{\mathrm{weak}}(\cdot|x_{0}) \propto p_{X_{0}|Y_{0}}^{\mathrm{weak}}(x_{0}|\cdot)p_{Y_{0}}, \quad x_{0}\in\mathcal{X}.
\end{equation}
The following corollary provides an expected TV-bound.

\begin{corollary}[Conditional generative error]\label{cor:conditional_generative_error}
Under the same assumptions as in Theorem~\ref{thm:error_bound}, let $p_{Y_{0}|X_{0}}^{\mathrm{weak}}$ be the weakly coupled conditional distribution defined in~\eqref{eq:condition:yx}. Then we have
\begin{align*}
&\mathbb{E}_{X_{0}}\bigl[\|p_{Y_0|X_0}(\cdot|X_0)-p_{Y_0|X_0}^{\mathrm{weak}}(\cdot|X_0)\|_{\mathrm{TV}}\bigr] \\
&\leq 2L_{\mathrm{post}}\sqrt{2d\frac{1-\bar{\alpha}_S}{\bar{\alpha}_S}}+2\zeta_{S}\frac{|\sigma_Y \sqrt{\bar{\alpha}_V} - \sigma_X \sqrt{\bar{\alpha}_S}|}{\sqrt{1-\bar{\alpha}_V}}+2\eta_{S}\sqrt{D_{\mathrm{KL}}(p_{\Xi} \,\|\, \mathcal{N}(0, \sigma_{\Xi}^2 I_d))}.
\end{align*}
\end{corollary}

\begin{proof}
    See Appendix~\ref{proof:error:condition}.
\end{proof}

\begin{remark}[Selection of time steps $S$ and $V$]\label{rem:time_steps}
The bounds motivate the variance-alignment condition $\bar{\alpha}_V\sigma_Y^2=\bar{\alpha}_S\sigma_X^2$, which eliminates the mean- and variance-mismatch terms. For additive Gaussian white noise, the non-Gaussianity term vanishes, while setting $S=0$ eliminates the perturbation term, recovering the exact formulation in Remark~\ref{rem:gaussian_weak_coupling}.

For practical noise, bounding the exact discrepancy is challenging. We therefore propose a simplified heuristic objective for selecting a suitable $S$. Specifically, we use $\sigma_X \approx \sigma_Y$ to simplify $\eta_S \approx \frac{\sqrt{2\bar{\alpha}_S}\sigma_{\Xi}}{\sigma_Y \sqrt{1-\bar{\alpha}_S}}$ and approximate the KL divergence as $D_{\mathrm{KL}}(p_{\Xi} \,\|\, \mathcal{N}(0, \sigma_{\Xi}^2 I_d))\approx \big|\kappa(Y)-\kappa(X)\big|^2$, which gives
\begin{equation}
    \bar{\alpha}_S^{*} = \underset{\bar{\alpha}_S \in (0, 1]}{\arg\min} \left(
    \lambda \sqrt{\frac{1-\bar{\alpha}_S}{\bar{\alpha}_S}} + \sqrt{\frac{\bar{\alpha}_S}{1-\bar{\alpha}_S}}\,\big|\kappa(Y)-\kappa(X)\big| \right) = \frac{\lambda}{\lambda + \big|\kappa(Y)-\kappa(X)\big|},
\end{equation}
where the hyperparameter $\lambda \coloneq \frac{L_{\mathrm{post}} \sigma_Y \sqrt{d}}{\sigma_{\Xi}}$ theoretically aggregates the constants from the error bound to balance the two components,
and $\kappa(\cdot)$ denotes the standardized dimension-averaged excess kurtosis, defined by
\begin{equation*}
    \kappa(X) \coloneq \frac{1}{d\sigma_X^4}\mathbb{E}\big[\|X - \mathbb{E}[X]\|_2^4\big] - (d+2).
\end{equation*}
Subsequently, the coupling step $V$ is determined by the variance alignment constraint $\bar{\alpha}_V = (\sigma_X^2 / \sigma_Y^2) \bar{\alpha}_S^*$. In practice, $\lambda$ can be empirically calibrated using synthetic noise data.
\end{remark}

\section{Experimental Results}\label{sec:experiments}

In this section, we evaluate the performance of LUD-DIF on different tasks. First, we assess its ability to model various simulated noise distributions on natural images. Next, we evaluate the model's capability in capturing complex real-world noise using a smartphone image denoising dataset (SIDD). Finally, we demonstrate the applicability of our method to image super-resolution.

Given a clean sample \(x_0\), we first diffuse it to the aligned state \(x_V\), and then use the learned conditional reverse process to generate the corresponding noisy or degraded observation \(y_0\). The generated pair \((x_0,y_0)\) is then used either for evaluating the learned degradation model or for training a downstream restoration network.

For implementation, we report the diffusion levels by their equivalent additive noise scales \(\tau_S\) and \(\tau_V\), rather than directly listing \(\bar{\alpha}_S\) and \(\bar{\alpha}_V\). On the original data scale, these quantities are related by
\[
\tau_S=\sigma_Y\sqrt{\frac{1-\bar{\alpha}_S}{\bar{\alpha}_S}},\qquad
\tau_V=\sigma_X\sqrt{\frac{1-\bar{\alpha}_V}{\bar{\alpha}_V}}.
\]
Under the variance-alignment condition \(\bar{\alpha}_V\sigma_Y^2=\bar{\alpha}_S\sigma_X^2\), this is equivalent to \(\tau_V^2=\tau_S^2+\sigma^2\), where \(\sigma\) denotes the standard deviation of the degradation noise.

\subsection{Simulated Image Noise}
\label{subsec:simu-noise}
\begin{figure}[t]
    \centering
    \includegraphics[width=1\textwidth, keepaspectratio]{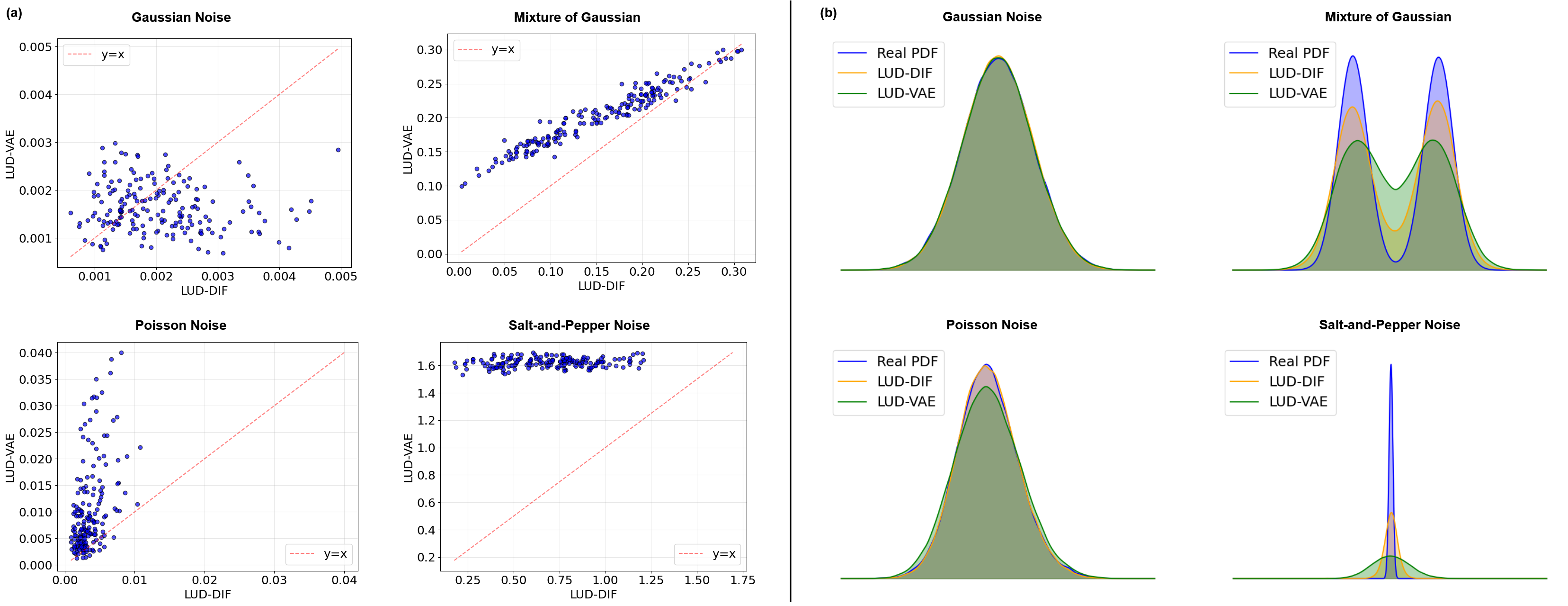}
    \caption{Comparison of the modeling capabilities of LUD-VAE and LUD-DIF for different simulated noises on the BSDS300 dataset. 
             (a) Scatter plot of the KL divergence between the generated noise and the ground-truth noise for each image in the validation set. 
             Points below the $y = x$ line indicate that LUD-DIF yields worse results on the corresponding sample. 
             (b) Comparison of the generated noise distribution curves between the two methods. 
             The selected sample points correspond to cases where both methods perform within the 25\%--50\% best range across all noise types.}
    \label{fig:simu-noise}
\end{figure}

\ 

\textbf{Dataset: } We evaluate the capability of LUD-DIF in simulating image noise distributions on the BSDS300 dataset \cite{MartinFTM01}. The BSDS300 dataset contains 300 natural images, of which 200 images are used as the training set and the remaining 100 images as the test set. For images in the training set, we conduct training on $64 \times 64$ patches, while for images in the test set, we perform evaluation on $256 \times 256$ patches. The training set is evenly split into two halves. One half is used to create four noisy datasets by adding Gaussian noise, Poisson noise, salt-and-pepper noise, and a mixture of Gaussian noises to achieve a PSNR of 20 dB. The other half serves as the unpaired clean dataset.

\textbf{Implementation Details: } Our model consists of a diffusion model for clean images and a conditional diffusion model for noisy images, both employing a UNet architecture. It is trained for 100k iterations using the Adam optimizer with an initial learning rate of $1 \times 10^{-4}$, which is gradually decayed to $1 \times 10^{-6}$ based on the loss reduction. The batch size is set to 8 during training. For sampling, we employ the DDIM sampler \cite{song_denoising_2022} with 50 steps and set the stochastic term $\eta$ to 0.0. For Gaussian noise, we set $\tau_S = 0$. For Poisson noise and mixed Gaussian noise, we set $\tau_S = 0.5\sigma$, where $\sigma$ is the standard deviation of the added noise, and choose $\tau_V = \sqrt{\tau_S^2 + \sigma^2}$ following the corresponding variance-matching rule in our implementation. For salt-and-pepper noise, we use the practical smoothing choice $\tau_V = \tau_S = \sigma$. Furthermore, when constructing the weak reconstruction \(p_{X_0|X_V}(\cdot|Y_S)\), we first smooth the input $y$ by replacing all pixels with values of 0 or 255 with the mean value of their neighboring pixels.

\textbf{Results: } We visualize the generation results of LUD-DIF under different noise distributions, as shown in Fig. \ref{fig:simu-noise} and Table \ref{tab:BSDS300}. 
It can be observed that for Gaussian noise, both methods demonstrate strong modeling capabilities. 
While LUD-VAE yields more stable generation, LUD-DIF exhibits larger errors on certain instances. 
For unimodal distributions such as Poisson noise, our method effectively captures the noise characteristics, 
generating noise images that closely approximate the real noise distribution. 
In contrast, LUD-VAE fails to capture the noise properties in a significant number of samples.

For multimodal structures like Gaussian mixture distributions, our method successfully captures the multimodal characteristics
but fails to precisely reconstruct the true noise distribution. This is attributed to the stronger non-Gaussianity and higher-order moment discrepancy of multimodal distributions, 
which is qualitatively consistent with the non-Gaussianity term in Theorem~\ref{thm:error_bound}. 
For discrete distributions such as salt-and-pepper noise, the smoothing technique described above improves the generation quality 
of the weak reconstruction \(p_{X_0|X_V}(\cdot|Y_S)\), thereby achieving better visual performance. 
However, since salt-and-pepper noise is zero-valued at most pixel locations and diffusion models struggle to guarantee identity mapping at these points,
the KL divergence remains relatively high. 
Conversely, LUD-VAE is unable to effectively model such extreme noise patterns.

\begin{table}[htbp]
\centering
\begin{tabular}{c|c|c|c|c}
\hline
Noise Type & Gaussian & Mixture Gaussian & Poisson & Salt-and-Pepper \\
\hline
LUD-VAE & \textbf{0.0017} & 0.2103 & 0.0147 & 1.6275 \\
LUD-DIF & 0.0023 & \textbf{0.1644} & \textbf{0.0043} & \textbf{0.7129} \\
\hline
\end{tabular}
\caption{Average KL divergence of LUD-VAE and LUD-DIF for different simulated noise types.}
\label{tab:BSDS300}
\end{table}
\subsection{Real-world Noise Modeling}

\ 

\textbf{Dataset:} We employ the Smartphone Image Denoising Dataset (SIDD) \cite{abdelhamed_high-quality_2018} to evaluate our method's ability to model complex real-world noise. The SIDD dataset contains noisy images and their corresponding clean images captured by different smartphones under various lighting conditions. We use the SIDD-Medium subset, which consists of 320 pairs of noisy and clean images. For each pair, we randomly crop 300 patches of size $256\times 256$, resulting in a total of 96,000 paired patches. When training the model, we use all 96,000 images but do not utilize their pairing information, meaning that each iteration uses either clean or noisy images. We refer to this training setup, in which paired data exist in the training set but the pairing relationships are unknown, as Weak-Unpaired (\textbf{WUP}). To simulate a strict unpaired setting, we split these 96,000 images into two halves: 48,000 images serve as the clean image set, and the remaining 48,000 images serve as the noisy image set. We refer to this setting, in which no paired data are available during training, as Strong-Unpaired (\textbf{SUP}). We use the SIDD validation set, which contains 1,280 images of size $256\times256$, for evaluation.

\begin{figure}[t]
    \centering
    \includegraphics[width=1\textwidth, keepaspectratio]{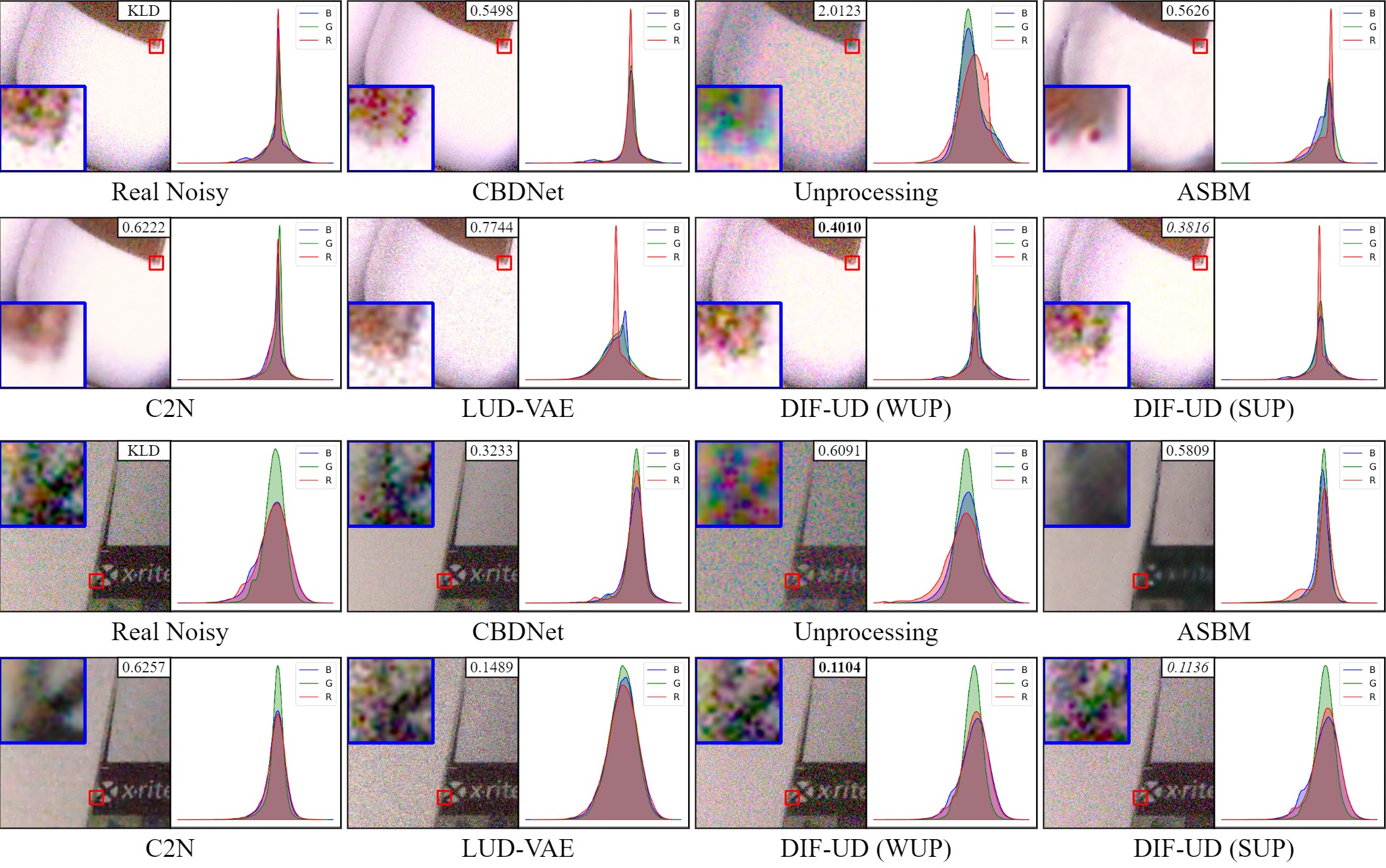}
    \caption{Comparison of different methods for unpaired noise modeling on the SIDD dataset.}
    \label{fig:real-noise}
\end{figure}
\textbf{Evaluation Metrics:} To assess the quality of the noise images generated by our method, we use the clean images from the SIDD validation set to generate their corresponding noisy images. We then evaluate the discrepancy between these generated noisy images and the real noisy images using the Fr\'echet Inception Distance (FID) \cite{heusel2017gans} and the Average KL Divergence (AKLD) between the synthesized and real noisy images. Motivated by the kurtosis-based heuristic in Remark~\ref{rem:time_steps}, we also report the Average Fourth-order Moment Difference (AFMD) to measure the difference in fourth-order statistics between the generated noise distribution and the real noise distribution:
\begin{align*}
    \text{AFMD} = \frac{1}{N}\sum_{i=1}^N \left| \kappa(n_{\text{real}}^{(i)}) - \kappa(n_{\text{fake}}^{(i)}) \right|
\end{align*}
where \(n_{\text{real}}^{(i)}\) and \(n_{\text{fake}}^{(i)}\) represent the real noise and the generated noise of the \(i\)-th sample, respectively, and \(\kappa\) denotes the standardized excess kurtosis of the noise. Additionally, we generate noisy images on a portion of the training set to obtain several pairs of noisy and clean images. These paired data are then used to train a DnCNN denoising network\footnote{More advanced denoisers often possess stronger generalization capabilities, which is not conducive to reflecting quality differences in the training set.}. The denoising performance is evaluated on the test set using Peak Signal-to-Noise Ratio (PSNR) and Structural Similarity Index Measure (SSIM) \cite{wang2004image}.

\textbf{Implementation Details:} The network architecture is the same as that described in Section~\ref{subsec:simu-noise}. The model is trained for 300k iterations using the Adam optimizer, with the learning rate set to $1 \times 10^{-4}$ and gradually decreased to $1 \times 10^{-6}$ based on the decline of the loss. During training, the batch size is set to 64, and the model is trained on randomly cropped image patches of size $64\times 64$. For sampling, the DDIM sampler is employed. To achieve better generation quality when generating noise for the validation set, we use 200 sampling steps and set the stochastic term $\eta$ to 0.5. For generating the training set for DnCNN, which prioritizes faster generation speed while ensuring noise randomness, we use 50 sampling steps and set $\eta$ to 1.0.
Due to the significant variation in noise distributions within the SIDD dataset, we select $V$ and $S$ separately for each image. Specifically, we first train a noise variance estimator on clean data using Gaussian-simulated noise. For a noisy image, this estimator is used to predict its noise variance $\hat{\sigma}^2$. The values of $V$ and $S$ are then selected by an empirically calibrated rule motivated by the heuristic objective in Remark~\ref{rem:time_steps}, with $V$ chosen to match the corresponding variance level after $S$ is fixed. Empirically, setting $\tau_S = 0.5 \sigma$ yields favorable results.
To enable the generation of noisy images with specified noise intensity, an additional channel is introduced in the input of the conditional diffusion model to feed in the noise level information. When generating noise simulations on the validation set, we directly use the given corresponding variance to generate the noisy images. Conversely, when generating noisy images on the training set for DnCNN training, we randomly sample the noise standard deviation from the interval $[5, 40]$ to produce the noisy images.
\begin{table}[t]
    \centering
    \caption{Quantitative comparison of noise generation quality on the SIDD validation set and denoising performance on the SIDD benchmark.
    The best results are highlighted in \textbf{bold}, and the second-best results are \underline{underlined}.
    Note that ``Paired'' refers to the UNet backbone performance using paired data.}
    \label{tab:sidd_results}
    \begin{tabular}{c|ccc|cc}
    \hline
    \multirow{2}{*}{Method} & \multicolumn{3}{c|}{Noise Generation Quality} & \multicolumn{2}{c}{Downstream Performance} \\
    \cline{2-6}
        & FID $\downarrow$ & AKLD $\downarrow$ & AFMD $\downarrow$ & PSNR $\uparrow$ & SSIM $\uparrow$ \\
    \hline
    CBDNet\cite{guo2019toward}                & 141.34 & 1.795 & 2.186 & 34.48 & 0.848 \\
    Unprocessing\cite{brooks2019unprocessing} & 95.10 & 0.964 & 1.338 & 21.74 & 0.440 \\
    \hline
    C2N\cite{jang2021c2n}                     & \underline{33.97} & 0.169 & 1.141 & 34.12 & 0.818 \\
    DeFlow\cite{wolf2021deflow}               & 39.45 & 0.205 & - & 33.82 & 0.846 \\
    LUD-VAE\cite{zheng_learn_2023}            & 35.31 & \underline{0.108} & \underline{0.780} & \underline{34.91} & \underline{0.892} \\
    ASBM\cite{gushchin_adversarial_2024}      & 68.93 & 0.552 & 0.848 & 34.32 & 0.866 \\
    \textbf{LUD-DIF (Ours)}                   & \textbf{16.99} & \textbf{0.046} & \textbf{0.420} & \textbf{35.47} & \textbf{0.896} \\
    \hline
    LUD-DIF(WUP)                              & 16.30 & 0.026 & 0.340 & 35.60 & 0.898 \\
    Paired                                   & 13.86 & 0.014 & 0.206 & 37.89 & 0.906 \\
    \hline
    \end{tabular}
\end{table}

\begin{table}[t]
    \centering
    \caption{Quantitative comparison on the SIDD Validation and SIDD Benchmark datasets.}
    \label{tab:sidd_down}
    \begin{tabular}{l|cc|cc}
    \hline
    \multirow{2}{*}{Method} & \multicolumn{2}{c|}{SIDD Validation} & \multicolumn{2}{c}{SIDD Benchmark} \\
    \cline{2-5}
    & PSNR $\uparrow$ & SSIM $\uparrow$ & PSNR $\uparrow$ & SSIM $\uparrow$ \\
    \hline
    N2V \cite{krull2019noise2void}            & 29.35 & 0.651 & 27.68 & 0.668 \\
    N2S \cite{batson2019noise2self}             & 30.72 & 0.787 & 29.56 & 0.808 \\
    CVF-SID \cite{neshatavar2022cvf}        & 34.17 & 0.872 & 34.71 & 0.917 \\
    AP-BSN + R$^3$ \cite{lee2022ap} & 35.91 & 0.882 & 35.97 & \underline{0.925} \\
    SCPGabNet \cite{lin2023unsupervised}      & 36.53 & 0.886 & 36.53 & 0.925 \\
    SDAP(S)(E) \cite{pan2023random}     & \textbf{37.55} & \underline{0.894} & \underline{37.53} & \textbf{0.936} \\
    \textbf{LUD-DIF (Ours)}        & \underline{37.37} & \textbf{0.906} & \textbf{37.59} & 0.898 \\
    \hline
    \end{tabular}
\end{table}

\textbf{Compared Methods:} We compare the LUD-DIF method with several unpaired noise modeling approaches, including C2N \cite{jang2021c2n}, DeFlow \cite{wolf2021deflow}, and LUD-VAE \cite{zheng_learn_2023}. These models are all trained using the hyperparameters specified in \cite{zheng_learn_2023} and employ the SUP training data. We also compare against two model-based noise modeling methods: CBDNet \cite{guo2019toward} and Unprocessing \cite{brooks2019unprocessing}. For EOT-based methods, we utilize ASBM \cite{gushchin_adversarial_2024}, which achieves state-of-the-art generation speed and quality. The diffusion term parameter for ASBM is set to 0.3 to match the overall noise intensity of the dataset. Furthermore, we evaluate the noise generation capability of the UNet backbone used in our method under a supervised setting labeled ``Paired''. We also compare our method in the WUP case to several unsupervised denoising methods, including N2V~\cite{krull2019noise2void}, N2S~\cite{batson2019noise2self},
CVF-SID~\cite{neshatavar2022cvf}, AP-BSN + R3~\cite{lee2022ap}, SCPGabNet~\cite{lin2023unsupervised}, and SDAP(S)(E)~\cite{pan2023random}. Here, we use DRUNet~\cite{zhang2021plug} as the downstream denoising network to get better results.

\textbf{Results:} The results are presented in Table \ref{tab:sidd_results}. Compared to other unpaired noise modeling methods, LUD-DIF achieves the best performance on most metrics. Notably, it shows substantial improvement across all three metrics for noisy image generation quality. The two model-based, non-learning noise modeling methods perform poorly in the noisy domain due to their inability to learn the underlying distributional characteristics of noise from data. The EOT-based ASBM method also yields subpar generation quality, which is attributed to error accumulation across different iteration steps. It is worth noting that the model trained on WUP data slightly outperforms the one trained on SUP data, indicating that implicit pairing within the dataset enhances the model's ability to capture data coupling. Furthermore, the performance gap between unpaired and paired training is qualitatively consistent with the weak-coupling analysis, since more accurate coupling information reduces the approximation error. Visual results, as shown in Figure \ref{fig:real-noise}, demonstrate that our method can more effectively characterize the spatial variation characteristics of noise.
\subsection{Image Super-Resolution}
\begin{table}[t]
 \centering
 \caption{Quantitative comparison of pseudo-degraded image generation quality in the noisy domain on the AIM19 and NTIRE20 datasets.}
 \label{tab:SR_noisy}
 \begin{tabular}{c|cc|cc}
 \hline
 \multirow{2}{*}{Method} & \multicolumn{2}{c|}{AIM19} & \multicolumn{2}{c}{NTIRE20} \\
 \cline{2-5}
 & AKLD $\downarrow$ & FID $\downarrow$ & AKLD $\downarrow$ & FID $\downarrow$ \\
 \hline
 Bicubic                 & 0.701 & 112.9 & 0.701 & 56.3 \\
 FSSR \cite{fritsche2019frequency}              & 0.379 & 67.6 & 0.224 & 40.2 \\
 Impressionism \cite{ji2020real}      & 0.549 & 87.3 & 0.275 & 26.5 \\
 DASR \cite{wei2021unsupervised}              & \underline{0.328} & 72.4 & 0.346 & 48.6 \\
 DeFlow \cite{wolf2021deflow}              & 0.356 & 119.8 & 0.156 & 34.7 \\
 LUD-VAE\cite{zheng_learn_2023} & 0.329 & \underline{57.2} & \underline{0.120} & \underline{24.9} \\
\textbf{LUD-DIF (Ours)}  & \textbf{0.271} & \textbf{56.6} & \textbf{0.106} & \textbf{24.2} \\
 \hline
 \end{tabular}
\end{table}

\ 

\textbf{Dataset:} We employ two unpaired super-resolution datasets, AIM19 and NTIRE20, to evaluate the applicability of LUD-DIF for image super-resolution tasks. Track 2 of AIM19 and Track 1 of NTIRE20 both contain several sets of unpaired high-resolution images and degraded images. We crop these into \(128 \times 128\) patches for training. Both datasets include a test set consisting of 100 paired low-resolution and high-resolution images, which we use to evaluate our method.

\textbf{Implementation Details:} The network architecture is the same as described in Section~\ref{subsec:simu-noise}. Since the exact forward downsampling operator for the super-resolution problem is unknown, we approximate it using bicubic interpolation. Specifically,
\begin{align*}
    y = D_{\text{bicubic}}(x) + n + D_{\text{real}}(x) - D_{\text{bicubic}}(x),
\end{align*}
where \(D_{\text{bicubic}}\) denotes the bicubic interpolation downsampling operator, \(D_{\text{real}}\) represents the true downsampling operator, and \(n\) is the real noise. We combine the real noise with the approximation error of the downsampling operator as an effective composite noise to be learned by the noise model; this is a practical extension of the additive formulation used in the analysis. Due to the low noise intensity in the super-resolution task, we directly model the diffusion process for \(x\), \(p_{X_0|X_V}\), as a single-step denoising model \(\mathbb{E}[X_0|X_V]\). In practice, for the AIM19 data, we set \(\tau_V = 15, \tau_S = 10\); for the NTIRE20 data, we set \(\tau_V = 8, \tau_S = 3\), consistent with the settings in \cite{zheng_learn_2023}.

\textbf{Results:} Table~\ref{tab:SR_noisy} compares the generated degradations with real degraded images using AKLD and FID. LUD-DIF achieves the best results on both metrics for the AIM19 and NTIRE20 datasets. For downstream evaluation, we train ESRGAN~\cite{wang2018esrgan} using the generated degraded images paired with their corresponding high-resolution images, and report PSNR, SSIM, and LPIPS in Table~\ref{tab:SR_clean}. LUD-DIF achieves competitive restoration performance, including the best LPIPS on both datasets and the best SSIM on NTIRE20.
\begin{table}[t]
    \centering
    \caption{Quantitative comparison on the AIM19 and NTIRE20 datasets in terms of PSNR, SSIM, and LPIPS.}
    \label{tab:SR_clean}
    \resizebox{\textwidth}{!}{
    \begin{tabular}{l|ccc|ccc}
    \hline
    \multirow{2}{*}{Method} & \multicolumn{3}{c|}{AIM19} & \multicolumn{3}{c}{NTIRE20} \\
    \cline{2-7}
    & PSNR $\uparrow$ & SSIM $\uparrow$ & LPIPS $\downarrow$ & PSNR $\uparrow$ & SSIM $\uparrow$ & LPIPS $\downarrow$ \\
    \hline
    Bicubic       & 21.69 & 0.5517 & 0.517 & 20.45 & 0.3241 & 0.675 \\
    FSSR \cite{fritsche2019frequency}            & 20.81 & 0.5242 & 0.387 & 21.07 & 0.4356 & 0.414 \\
    Impressionism \cite{ji2020real}   & 21.99 & \underline{0.6060} & 0.420 & 25.27 & 0.6731 & 0.229 \\
    DASR \cite{wei2021unsupervised}            & 21.06 & 0.5658 & 0.375 & 23.70 & 0.5748 & 0.328 \\
    DeFlow \cite{wolf2021deflow}        & 21.06 & 0.5842 & 0.346 & 24.81 & 0.6777 & 0.225 \\
    LUD-VAE\cite{zheng_learn_2023} & \textbf{22.32} & \textbf{0.6197} & \underline{0.341} & \textbf{25.79} & \underline{0.7178} & \underline{0.219} \\
    \textbf{LUD-DIF (Ours)} & \underline{22.09} & 0.6046 & \textbf{0.332} & \underline{25.77} & \textbf{0.7200} & \textbf{0.214} \\
    \hline
    \end{tabular}
    }
\end{table}

\subsection{Ablation Study}
\label{subsec:ablation}
Appendix~\ref{supp:subsec:independence-test} tests the independence of generated noise samples. Appendix~\ref{supp:subsec:hyperparameter-selection} studies training and sampling hyperparameters. Appendix~\ref{supp:subsec:degradation-level} examines the effect of the degradation level \(\tau_S\), while Appendix~\ref{supp:subsec:denoising-steps} analyzes the number of denoising steps used in the weak reconstruction. Finally, Appendix~\ref{supp:subsec:two-path-noise-modeling} studies the role of the two-path diffusion architecture in noise modeling. Overall, these results support the practical choices used in the main experiments and are consistent with the weak-coupling trade-off described in Section~\ref{sec:analysis}.

\section{Conclusion and Future Work}\label{sec5}
This paper proposes LUD-DIF, a diffusion-based method for unpaired data learning in inverse problems. Under the weak-coupling assumption, the joint ELBO is decoupled into two independently trainable diffusion processes; without this assumption, we quantitatively analyze the error bound and provide a theorem-motivated heuristic for hyperparameter selection. Experiments on simulated and real-world image noise and image super-resolution demonstrate the effectiveness of the method.

Despite these results, LUD-DIF has two main limitations. First, it requires a known or approximate forward operator, limiting its use in blind inverse problems and semantic modality translation. Second, although we bound the error introduced by the weak-coupling assumption, how to reduce this error remains open. Future work may investigate transformed or latent representations, such as wavelet spaces and neural-network embeddings, to tighten the bound, together with efficient sampling and model distillation to improve scalability.

\bibliographystyle{plain}
\bibliography{LUD-DIF}

\begin{thebibliography}{10}

\bibitem{abdelhamed_high-quality_2018}
Abdelrahman Abdelhamed, Stephen Lin, and Michael~S. Brown.
\newblock A high-quality denoising dataset for smartphone cameras.
\newblock In {\em 2018 IEEE/CVF Conference on Computer Vision and Pattern
  Recognition}, pages 1692--1700, 2018.

\bibitem{adler2021deep}
Amir Adler, Mauricio Araya-Polo, and Tomaso Poggio.
\newblock Deep learning for seismic inverse problems: {T}oward the acceleration
  of geophysical analysis workflows.
\newblock {\em IEEE signal processing magazine}, 38(2):89--119, 2021.

\bibitem{arhire2025learned}
Andrei Arhire and Radu Timofte.
\newblock Learned lightweight smartphone {ISP} with unpaired data.
\newblock In {\em Proceedings of the Computer Vision and Pattern Recognition
  Conference ({CVPR}) Workshops}, pages 1878--1887, 2025.

\bibitem{Bakry1985Diffusions}
Dominique Bakry and Michael {\'E}mery.
\newblock Diffusions hypercontractives.
\newblock In Jacques Az{\'e}ma and Marc Yor, editors, {\em S{\'e}minaire de
  Probabilit{\'e}s XIX 1983/84}, pages 177--206. Springer Berlin Heidelberg,
  1985.

\bibitem{batson2019noise2self}
Joshua Batson and Loic Royer.
\newblock Noise2self: {B}lind denoising by self-supervision.
\newblock In {\em International conference on machine learning}, pages
  524--533. PMLR, 2019.

\bibitem{bertero2006inverse}
Mario Bertero and Michele Piana.
\newblock Inverse problems in biomedical imaging: Modeling and methods of
  solution.
\newblock In Alfio Quarteroni, Luca Formaggia, and Alessandro Veneziani,
  editors, {\em Complex Systems in Biomedicine}, pages 1--33. Springer Milan,
  2006.

\bibitem{brooks2019unprocessing}
Tim Brooks, Ben Mildenhall, Tianfan Xue, Jiawen Chen, Dillon Sharlet, and
  Jonathan~T Barron.
\newblock Unprocessing images for learned raw denoising.
\newblock In {\em Proceedings of the IEEE/CVF conference on computer vision and
  pattern recognition}, pages 11036--11045, 2019.

\bibitem{daras2024survey}
Giannis Daras, Hyungjin Chung, Chieh-Hsin Lai, Yuki Mitsufuji, Jong~Chul Ye,
  Peyman Milanfar, Alexandros~G Dimakis, and Mauricio Delbracio.
\newblock A survey on diffusion models for inverse problems.
\newblock {\em arXiv:2410.00083}, 2024.

\bibitem{de_bortoli_schrodinger_2024}
Valentin De~Bortoli, Iryna Korshunova, Andriy Mnih, and Arnaud Doucet.
\newblock Schr{\"o}dinger bridge flow for unpaired data translation.
\newblock In {\em Advances in Neural Information Processing Systems},
  volume~37, pages 103384--103441, 2024.

\bibitem{fritsche2019frequency}
Manuel Fritsche, Shuhang Gu, and Radu Timofte.
\newblock Frequency separation for real-world super-resolution.
\newblock In {\em 2019 IEEE/CVF International Conference on Computer Vision
  Workshop (ICCVW)}, pages 3599--3608. IEEE, 2019.

\bibitem{genzel2022near}
Martin Genzel, Ingo G{\"u}hring, Jan Macdonald, and Maximilian M{\"a}rz.
\newblock Near-exact recovery for tomographic inverse problems via deep
  learning.
\newblock In {\em International Conference on Machine Learning}, pages
  7368--7381. PMLR, 2022.

\bibitem{gu_optimal_2023}
Xiang Gu, Liwei Yang, Jian Sun, and Zongben Xu.
\newblock Optimal transport-guided conditional score-based diffusion model.
\newblock In {\em Advances in Neural Information Processing Systems},
  volume~36, 2023.

\bibitem{guo2019toward}
Shi Guo, Zifei Yan, Kai Zhang, Wangmeng Zuo, and Lei Zhang.
\newblock Toward convolutional blind denoising of real photographs.
\newblock In {\em Proceedings of the IEEE/CVF conference on computer vision and
  pattern recognition}, pages 1712--1722, 2019.

\bibitem{gushchin_adversarial_2024}
Nikita Gushchin, Daniil Selikhanovych, Sergei Kholkin, Evgeny Burnaev, and
  Alexander Korotin.
\newblock Adversarial {S}chr{\"o}dinger bridge matching.
\newblock {\em Advances in Neural Information Processing Systems},
  37:89612--89651, 2024.

\bibitem{heusel2017gans}
Martin Heusel, Hubert Ramsauer, Thomas Unterthiner, Bernhard Nessler, and Sepp
  Hochreiter.
\newblock {GANs} trained by a two time-scale update rule converge to a local
  {Nash} equilibrium.
\newblock {\em Advances in neural information processing systems}, 30, 2017.

\bibitem{ho_denoising_2020}
Jonathan Ho, Ajay Jain, and Pieter Abbeel.
\newblock Denoising diffusion probabilistic models.
\newblock In {\em Advances in Neural Information Processing Systems},
  volume~33, pages 6840--6851. Curran Associates, Inc., 2020.

\bibitem{jang2021c2n}
Geonwoon Jang, Wooseok Lee, Sanghyun Son, and Kyoung~Mu Lee.
\newblock {C2N}: {P}ractical generative noise modeling for real-world
  denoising.
\newblock In {\em Proceedings of the IEEE/CVF International Conference on
  Computer Vision}, pages 2350--2359, 2021.

\bibitem{ji2020real}
Xiaozhong Ji, Yun Cao, Ying Tai, Chengjie Wang, Jilin Li, and Feiyue Huang.
\newblock Real-world super-resolution via kernel estimation and noise
  injection.
\newblock In {\em proceedings of the IEEE/CVF conference on computer vision and
  pattern recognition workshops}, pages 466--467, 2020.

\bibitem{kabanikhin2011inverse}
Sergey~I Kabanikhin.
\newblock {\em Inverse and ill-posed problems: theory and applications}.
\newblock de Gruyter, 2011.

\bibitem{kamyab2022deep}
Shima Kamyab, Zohreh Azimifar, Rasool Sabzi, and Paul Fieguth.
\newblock Deep learning methods for inverse problems.
\newblock {\em PeerJ Computer Science}, 8:e951, 2022.

\bibitem{kim2023unpaired}
Beomsu Kim, Gihyun Kwon, Kwanyoung Kim, and Jong~Chul Ye.
\newblock Unpaired image-to-image translation via neural {S}chr\"odinger
  bridge.
\newblock In {\em The Twelfth International Conference on Learning
  Representations}, 2024.

\bibitem{krull2019noise2void}
Alexander Krull, Tim-Oliver Buchholz, and Florian Jug.
\newblock Noise2void-learning denoising from single noisy images.
\newblock In {\em Proceedings of the IEEE/CVF conference on computer vision and
  pattern recognition}, pages 2129--2137, 2019.

\bibitem{lee2022ap}
Wooseok Lee, Sanghyun Son, and Kyoung~Mu Lee.
\newblock {AP-BSN}: {S}elf-supervised denoising for real-world images via
  asymmetric pd and blind-spot network.
\newblock In {\em Proceedings of the IEEE/CVF Conference on Computer Vision and
  Pattern Recognition}, pages 17725--17734, 2022.

\bibitem{lehtinen2018noise2noise}
Jaakko Lehtinen, Jacob Munkberg, Jon Hasselgren, Samuli Laine, Tero Karras,
  Miika Aittala, and Timo Aila.
\newblock {N}oise2{N}oise: Learning image restoration without clean data.
\newblock In Jennifer Dy and Andreas Krause, editors, {\em Proceedings of the
  35th International Conference on Machine Learning}, volume~80 of {\em
  Proceedings of Machine Learning Research}, pages 2965--2974. PMLR, 2018.

\bibitem{li2019asymmetric}
Yu~Li, Sheng Tang, Rui Zhang, Yongdong Zhang, Jintao Li, and Shuicheng Yan.
\newblock Asymmetric {GAN} for unpaired image-to-image translation.
\newblock {\em IEEE Transactions on Image Processing}, 28(12):5881--5896, 2019.

\bibitem{lin2023unsupervised}
Xin Lin, Chao Ren, Xiao Liu, Jie Huang, and Yinjie Lei.
\newblock Unsupervised image denoising in real-world scenarios via
  self-collaboration parallel generative adversarial branches.
\newblock In {\em Proceedings of the IEEE/CVF International Conference on
  Computer Vision}, pages 12642--12652, 2023.

\bibitem{lin_decoupled_2026}
Ziyue Lin, Jiahe Hou, Hongyu Xia, Xinrui Xie, Feifei Wang, Yuyin Zhou, Wei
  Wang, Jiawei Liu, and Liangqiong Qu.
\newblock Decoupled residual denoising diffusion models for unified and data
  efficient image-to-image translation.
\newblock In {\em Proceedings of the IEEE/CVF Conference on Computer Vision and
  Pattern Recognition}, pages 35967--35977, 2026.

\bibitem{liu_generalized_2024}
Guan-Horng Liu, Yaron Lipman, Maximilian Nickel, Brian Karrer, Evangelos
  Theodorou, and Ricky T.~Q. Chen.
\newblock Generalized {S}chr\"odinger bridge matching.
\newblock In {\em The Twelfth International Conference on Learning
  Representations}, 2024.

\bibitem{MartinFTM01}
D.~Martin, C.~Fowlkes, D.~Tal, and J.~Malik.
\newblock A database of human segmented natural images and its application to
  evaluating segmentation algorithms and measuring ecological statistics.
\newblock In {\em Proc. 8th Int'l Conf. Computer Vision}, volume~2, pages
  416--423, July 2001.

\bibitem{meanti_unsupervised_2025}
Giacomo Meanti, Thomas Ryckeboer, Michael Arbel, and Julien Mairal.
\newblock Unsupervised imaging inverse problems with diffusion distribution
  matching.
\newblock In {\em Proceedings of the IEEE/CVF International Conference on
  Computer Vision}, pages 28364--28374, 2025.

\bibitem{neshatavar2022cvf}
Reyhaneh Neshatavar, Mohsen Yavartanoo, Sanghyun Son, and Kyoung~Mu Lee.
\newblock {CVF-SID}: {C}yclic multi-variate function for self-supervised image
  denoising by disentangling noise from image.
\newblock In {\em Proceedings of the IEEE/CVF Conference on Computer Vision and
  Pattern Recognition}, pages 17583--17591, 2022.

\bibitem{ongie2020deep}
Gregory Ongie, Ajil Jalal, Christopher~A Metzler, Richard~G Baraniuk,
  Alexandros~G Dimakis, and Rebecca Willett.
\newblock Deep learning techniques for inverse problems in imaging.
\newblock {\em IEEE Journal on Selected Areas in Information Theory},
  1(1):39--56, 2020.

\bibitem{pan2023random}
Yizhong Pan, Xiao Liu, Xiangyu Liao, Yuanzhouhan Cao, and Chao Ren.
\newblock Random sub-samples generation for self-supervised real image
  denoising.
\newblock In {\em Proceedings of the IEEE/CVF international conference on
  computer vision}, pages 12150--12159, 2023.

\bibitem{peng_towards_2025}
Long Peng, Wenbo Li, Renjing Pei, Jingjing Ren, Jiaqi Xu, Yang Wang, Yang Cao,
  and Zheng-Jun Zha.
\newblock Towards realistic data generation for real-world super-resolution.
\newblock In {\em The Thirteenth International Conference on Learning
  Representations}, 2025.

\bibitem{ribes2008linear}
Alejandro Ribes and Francis Schmitt.
\newblock Linear inverse problems in imaging.
\newblock {\em IEEE Signal Processing Magazine}, 25(4):84--99, 2008.

\bibitem{song_denoising_2022}
Jiaming Song, Chenlin Meng, and Stefano Ermon.
\newblock Denoising diffusion implicit models.
\newblock In {\em International Conference on Learning Representations}, 2021.

\bibitem{song2020score}
Yang Song, Jascha Sohl-Dickstein, Diederik~P Kingma, Abhishek Kumar, Stefano
  Ermon, and Ben Poole.
\newblock Score-based generative modeling through stochastic differential
  equations.
\newblock In {\em International Conference on Learning Representations}, 2021.

\bibitem{spantini2015optimal}
Alessio Spantini, Antti Solonen, Tiangang Cui, James Martin, Luis Tenorio, and
  Youssef Marzouk.
\newblock Optimal low-rank approximations of {B}ayesian linear inverse
  problems.
\newblock {\em SIAM Journal on Scientific Computing}, 37(6):A2451--A2487, 2015.

\bibitem{stuart2010inverse}
Andrew~M Stuart.
\newblock Inverse problems: a {B}ayesian perspective.
\newblock {\em Acta numerica}, 19:451--559, 2010.

\bibitem{tropp2010computational}
Joel~A Tropp and Stephen~J Wright.
\newblock Computational methods for sparse solution of linear inverse problems.
\newblock {\em Proceedings of the IEEE}, 98(6):948--958, 2010.

\bibitem{wang2018esrgan}
Xintao Wang, Ke~Yu, Shixiang Wu, Jinjin Gu, Yihao Liu, Chao Dong, Yu~Qiao, and
  Chen~Change Loy.
\newblock {ESRGAN}: {E}nhanced super-resolution generative adversarial
  networks.
\newblock In {\em Computer Vision -- {ECCV} 2018 Workshops}, pages 63--79.
  Springer, 2019.

\bibitem{wang2004image}
Zhou Wang, Alan~C Bovik, Hamid~R Sheikh, and Eero~P Simoncelli.
\newblock Image quality assessment: from error visibility to structural
  similarity.
\newblock {\em IEEE transactions on image processing}, 13(4):600--612, 2004.

\bibitem{wei2020semi}
Yanyan Wei, Zhao Zhang, Yang Wang, Haijun Zhang, Mingbo Zhao, Mingliang Xu, and
  Meng Wang.
\newblock Semi-deraingan: A new semi-supervised single image deraining.
\newblock In {\em 2021 IEEE International Conference on Multimedia and Expo
  (ICME)}, pages 1--6, 2021.

\bibitem{wei2021unsupervised}
Yunxuan Wei, Shuhang Gu, Yawei Li, Radu Timofte, Longcun Jin, and Hengjie Song.
\newblock Unsupervised real-world image super resolution via domain-distance
  aware training.
\newblock In {\em Proceedings of the IEEE/CVF conference on computer vision and
  pattern recognition}, pages 13385--13394, 2021.

\bibitem{wiggins1972general}
Ralph~A Wiggins.
\newblock The general linear inverse problem: {I}mplication of surface waves
  and free oscillations for earth structure.
\newblock {\em Reviews of Geophysics}, 10(1):251--285, 1972.

\bibitem{wolf2021deflow}
Valentin Wolf, Andreas Lugmayr, Martin Danelljan, Luc Van~Gool, and Radu
  Timofte.
\newblock {DeFlow}: {L}earning complex image degradations from unpaired data
  with conditional flows.
\newblock In {\em Proceedings of the IEEE/CVF Conference on Computer Vision and
  Pattern Recognition}, pages 94--103, 2021.

\bibitem{zhang2021plug}
Kai Zhang, Yawei Li, Wangmeng Zuo, Lei Zhang, Luc Van~Gool, and Radu Timofte.
\newblock Plug-and-play image restoration with deep denoiser prior.
\newblock {\em IEEE Transactions on Pattern Analysis and Machine Intelligence},
  44(10):6360--6376, 2021.

\bibitem{zheng_learn_2023}
Dihan Zheng, Xiaowen Zhang, Kaisheng Ma, and Chenglong Bao.
\newblock Learn from unpaired data for image restoration: {A} variational
  {B}ayes approach.
\newblock {\em IEEE Transactions on Pattern Analysis and Machine Intelligence},
  45(5):5889--5903, 2023.

\end{thebibliography}

\appendix
\section{Theoretical details in LUD-DIF}\label{secA1}
\subsection{Proof of Proposition \ref{prop:factorization}}
\label{proof:factorization}
Since the noises in the forward processes are independent, $X_{t}$ and $Y_{t}$ are conditionally independent given $(X_{0},Y_{0})$ for any $t$. Hence, \eqref{eq:prop-forward} directly holds. Similarly, the transition probability factorizes as
\begin{equation}
p_{Z_{t}|Z_{t-1}}(z_{t}|z_{t-1}) = p_{X_{t}|X_{t-1}}(x_{t}|x_{t-1})p_{Y_{t}|Y_{t-1}}(y_{t}|y_{t-1}). \label{eq:forward:decoupling:2}
\end{equation}
Then it follows that
\begin{align*}
p_{Z_{t-1}|Z_{t},Z_{0}}(z_{t-1}|z_{t},z_{0})
&=\frac{p_{Z_{t}|Z_{t-1},Z_{0}}(z_{t}|z_{t-1},z_{0})p_{Z_{t-1}|Z_{0}}(z_{t-1}|z_{0})}{p_{Z_{t}|Z_{0}}(z_{t}|z_{0})} \\
&=\frac{p_{Z_{t}|Z_{t-1}}(z_{t}|z_{t-1})p_{Z_{t-1}|Z_{0}}(z_{t-1}|z_{0})}{p_{Z_{t}|Z_{0}}(z_{t}|z_{0})} \\
&=\frac{p_{X_{t}|X_{t-1}}(x_{t}|x_{t-1})p_{X_{t-1}|X_{0}}(x_{t-1}|x_{0})}{p_{X_{t}|X_{0}}(x_{t}|x_{0})}\frac{p_{Y_{t}|Y_{t-1}}(y_{t}|y_{t-1})p_{Y_{t-1}|Y_{0}}(y_{t-1}|y_{0})}{p_{Y_{t}|Y_{0}}(y_{t}|y_{0})} \\
&=p_{X_{t-1}|X_{t},X_{0}}(x_{t-1}|x_{t},x_{0})p_{Y_{t-1}|Y_{t},Y_{0}}(y_{t-1}|y_{t},y_{0}),
\end{align*}
where the first equality is due to Bayes' rule, the second equality invokes the Markov property of $Z_{t}$, the third equality follows from~\eqref{eq:prop-forward} and~\eqref{eq:forward:decoupling:2}, and the last equality follows from Bayes' rule.

\subsection{Proof of Proposition \ref{proposition:objective:groundtruth}}
\label{proof:objective:groundtruth}
Using arguments similar to those in the marginal derivation in~\eqref{eq:elbo:diffusion}, we expand the joint ELBO:
\begin{equation}\label{eq:proposition:objective:groundtruth:1}
\begin{aligned}
\mathrm{ELBO}_{Z}(\theta;z_{0}) =
& -\overbrace{D_{\mathrm{KL}}(p_{Z_{T}|Z_{0}}(\cdot|z_{0})\|p_{Z_{T}})}^{A_{1}}+\mathbb{E}\Big[\overbrace{\log p_{Z_{0}|Z_{1}}^{\theta}(z_{0}|Z_{1})}^{A_{2}}|Z_{0}=z_{0}\Big] \\
&- \sum_{t=2}^{T}\mathbb{E}\Big[\underbrace{D_{\mathrm{KL}}\big(p_{Z_{t-1}|Z_{t},Z_{0}}(\cdot|Z_{t},z_{0})\|p_{Z_{t-1}|Z_{t}}^{\theta}(\cdot|Z_{t})\big)}_{A_{3}}|Z_{0}=z_{0}\Big].
\end{aligned}
\end{equation}
The term $A_{1}$ in~\eqref{eq:proposition:objective:groundtruth:1} is a constant independent of $\theta$. For the term $A_{2}$, it follows directly from the reverse factorization~\eqref{eq:prop-reverse} that 
\begin{equation}\label{eq:proposition:objective:groundtruth:2}
\log p_{Z_{0}|Z_{1}}^{\theta}(z_0|Z_{1})=\log p_{X_{0}|X_{1}}^{\theta}(x_{0}|X_{1}) + \log p_{Y_{0} \mid X_{0},Y_{1}}^{\theta}(y_{0} \mid x_{0},Y_{1}).
\end{equation}
For the KL divergence term $A_{3}$ in~\eqref{eq:proposition:objective:groundtruth:1}, we apply both the forward posterior factorization~\eqref{eq:prop-posterior} and the reverse factorization~\eqref{eq:prop-reverse}:
\begin{align}
&D_{\rm KL}(p_{Z_{t-1}|Z_{t},Z_{0}}(\cdot|Z_{t},z_{0})\|p_{Z_{t-1}|Z_{t}}^{\theta}(\cdot|Z_{t})) \nonumber \\
&= \mathbb{E}\Bigg[\log\frac{p_{Z_{t-1}|Z_{t},Z_{0}}(Z_{t-1}|Z_{t},z_{0})}{p_{Z_{t-1}|Z_{t}}^{\theta}(Z_{t-1}|Z_{t})} \Bigg| Z_{t},Z_{0}=z_{0}\Bigg] \nonumber \\
&= \mathbb{E}\Bigg[\log\frac{p_{X_{t-1}|X_{t},X_{0}}(X_{t-1}|X_{t},x_{0})p_{Y_{t-1}|Y_{t},Y_{0}}(Y_{t-1}|Y_{t},y_{0})}{p_{X_{t-1} \mid X_{t}}^{\theta}(X_{t-1} \mid X_{t})p_{Y_{t-1} \mid X_{t-1},Y_{t}}^{\theta}(Y_{t-1} \mid X_{t-1},Y_{t})}\Bigg| Z_{t},Z_{0}=z_{0}\Bigg] \nonumber \\
&= \mathbb{E}\Bigg[\log\frac{p_{X_{t-1}|X_{t},X_{0}}(X_{t-1}|X_{t},x_{0})}{p_{X_{t-1}|X_{t}}^{\theta}(X_{t-1}|X_{t})}\Bigg| X_{t},X_{0}=x_{0}\Bigg] \nonumber \\
&\quad +\mathbb{E}\Bigg[\log\frac{p_{Y_{t-1}|Y_{t},Y_{0}}(Y_{t-1}|Y_{t},y_{0})}{p_{Y_{t-1} \mid X_{t-1},Y_{t}}^{\theta}(Y_{t-1} \mid X_{t-1},Y_{t})}\Bigg| Z_{t},Z_{0}=z_{0}\Bigg] \nonumber \\
&= D_{\rm KL}(p_{X_{t-1}|X_{t},X_{0}}(\cdot|X_{t},x_{0}) \| p_{X_{t-1}|X_{t}}^{\theta}(\cdot|X_{t})) \nonumber \\
&\quad + \mathbb{E}\left[D_{\rm KL}\bigl(p_{Y_{t-1}|Y_{t},Y_{0}}(\cdot|Y_{t},y_{0}) \big\| p_{Y_{t-1} \mid X_{t-1},Y_{t}}^{\theta}(\cdot \mid X_{t-1},Y_{t})\bigr) | Z_{t},Z_{0}=z_{0}\right], \label{eq:proposition:objective:groundtruth:3}
\end{align}
where the second equality follows from~\eqref{eq:prop-posterior} and~\eqref{eq:prop-reverse}. Substituting~\eqref{eq:proposition:objective:groundtruth:2} and~\eqref{eq:proposition:objective:groundtruth:3} into~\eqref{eq:proposition:objective:groundtruth:1} yields
\begin{align*}
\mathrm{ELBO}_{Z}(\theta;z_{0})
&= \mathrm{ELBO}_{X}(\theta;x_{0}) + \mathrm{ELBO}_{Y \mid X}(\theta;z_{0}) + C(z_{0}),
\end{align*}
where $C(z_{0})\coloneq -D_{\mathrm{KL}}(p_{Z_{T}|Z_{0}}(\cdot|z_{0})\|p_{Z_{T}})$ is independent of $\theta$ and can be evaluated using the standard Gaussian KL formula. This completes the proof.

\subsection{Proof of Theorem~\ref{theorem:equivalence}}
\label{proof:equivalence}

\begin{proof}[Proof of Theorem~\ref{theorem:equivalence}]
Comparing the objectives in~\eqref{eq:joint:objective} and~\eqref{eq:objective:weak}, it suffices to show
\begin{equation*}
(X_{t-1}^{\rm weak}|Y_{t},Y_{0}) \stackrel{\mathrm{d}}{=} (X_{t-1}|Y_{t},Y_{0}),
\end{equation*}
where the weakly coupled variable is defined as
\begin{equation}\label{eq:theorem:equivalence:1}
X_{t-1}^{\rm weak}\coloneq\sqrt{\bar{\alpha}_{t-1}}X_{0}^{\rm weak}+\sqrt{1-\bar{\alpha}_{t-1}}\xi.
\end{equation}
According to the Chapman--Kolmogorov equation, we have
\begin{align}
&p_{X_{t-1}^{\rm weak}|Y_{t},Y_{0}}(x_{t-1}|y_{t},y_{0}) \nonumber \\
&=\int \underbrace{p_{X_{t-1}^{\rm weak}|X_{0}^{\rm weak},Y_{t},Y_{0}}(x_{t-1}|x_{0},y_{t},y_{0})}_{A_{1}}\underbrace{p_{X_{0}^{\rm weak}|Y_{t},Y_{0}}(x_{0}|y_{t},y_{0})}_{A_{2}}\,\mathrm{d}x_{0}. \label{eq:theorem:equivalence:2}
\end{align} 
We first focus on the term $A_{1}$ in~\eqref{eq:theorem:equivalence:2}. By the definition of the weakly coupled variable in~\eqref{eq:theorem:equivalence:1}, $X_{t-1}^{\rm weak}$ is determined by $X_{0}^{\rm weak}$ and an independent variable $\xi$. Therefore, $X_{t-1}^{\rm weak}$ is conditionally independent of $(Y_{t},Y_{0})$ given $X_{0}^{\rm weak}$. Moreover, the definition in~\eqref{eq:theorem:equivalence:1} shows that the process $X_{0:T}^{\rm weak}$ shares the same forward process as $X_{0:T}$ in~\eqref{eq:forward:X:t}, i.e., $p_{X_{t-1}^{\rm weak}|X_{0}^{\rm weak}}=p_{X_{t-1}|X_{0}}$. Consequently, we have
\begin{equation}\label{eq:theorem:equivalence:3}
p_{X_{t-1}^{\rm weak}|X_{0}^{\rm weak},Y_{t},Y_{0}}(x_{t-1}|x_{0},y_{t},y_{0}) = p_{X_{t-1}^{\rm weak}|X_{0}^{\rm weak}}(x_{t-1}|x_{0}) = p_{X_{t-1}|X_{0}}(x_{t-1}|x_{0}).
\end{equation}
We next consider the term $A_{2}$ in~\eqref{eq:theorem:equivalence:2}. From~\eqref{eq:forward:joint}, conditioning on a fixed $Y_{0}=y_{0}$, the variable $Y_{t}$ is determined by $\varepsilon_{1:t-1}^{Y}$, which is independent of $X_{0}^{\rm weak}$. This implies
\begin{equation}\label{eq:theorem:equivalence:4}
p_{X_{0}^{\rm weak}|Y_{t},Y_{0}}(x_{0}|y_{t},y_{0})=p_{X_{0}^{\rm weak}|Y_{0}}(x_{0}|y_{0})=p_{X_{0}|Y_{0}}^{\rm weak}(x_{0}|y_{0})=p_{X_{0}|Y_{0}}(x_{0}|y_{0}),
\end{equation}
where the second equality is owing to the definition of $X_{0}^{\rm weak}$ in~\eqref{eq:weak:coupling:denoising}, and the last equality invokes Assumption~\ref{assum:weak:coupling}. Substituting~\eqref{eq:theorem:equivalence:3} and~\eqref{eq:theorem:equivalence:4} into~\eqref{eq:theorem:equivalence:2} yields
\begin{align*}
p_{X_{t-1}^{\rm weak}|Y_{t},Y_{0}}(x_{t-1}|y_{t},y_{0})
&=\int p_{X_{t-1}|X_{0}}(x_{t-1}|x_{0})p_{X_{0}|Y_{0}}(x_{0}|y_{0})\,\mathrm{d}x_{0} \\
&=\int p_{X_{t-1}|X_{0},Y_{t},Y_{0}}(x_{t-1}|x_{0},y_{t},y_{0})p_{X_{0}|Y_{t},Y_{0}}(x_{0}|y_{t},y_{0})\,\mathrm{d}x_{0} \\
&=p_{X_{t-1}|Y_{t},Y_{0}}(x_{t-1}|y_{t},y_{0}),
\end{align*}
where the second equality uses the facts that $X_{t-1}$ is conditionally independent of $(Y_{t},Y_{0})$ given $X_{0}$ and that $X_{0}$ is conditionally independent of $Y_{t}$ given $Y_{0}$, while the last equality follows from the Chapman--Kolmogorov equation. This completes the proof.
\end{proof}

\subsection{Proof of Proposition~\ref{proposition:delta}}
\label{proof:delta}
\begin{proof}[Proof of Proposition~\ref{proposition:delta}]
Define an auxiliary function 
\begin{equation*}
D_t(x_0, Y_t, Y_0, \varepsilon^X) \coloneq \|\mu_{Y\mid X,t}^{\theta}(Y_{t}, \sqrt{\bar{\alpha}_{t-1}}x_{0}+\sqrt{1-\bar{\alpha}_{t-1}}\varepsilon^{X})-\tilde{\mu}_{t}(Y_{t},Y_{0})\|_{2}^{2}.
\end{equation*}
Using~\eqref{eq:joint:objective} and~\eqref{eq:objective:weak}, the triangle inequality yields
\begin{equation}\label{eq:proposition:delta:1}
\Delta \leq \sum_{t=1}^{T}\frac{1}{2\sigma_{t}^{2}}\left|\mathbb{E}\Bigl[D_t(X_{0}, Y_t, Y_0, \varepsilon^X)\Bigr] - \mathbb{E}\Bigl[D_t(X_{0}^{\mathrm{weak}}, Y_t, Y_0, \varepsilon^X) \Bigr]\right|, 
\end{equation}
where $(X_{0}^{\mathrm{weak}}|Y_{0}=y_{0})\sim p_{X_{0}|Y_{0}}^{\mathrm{weak}}(\cdot|y_{0})$ and $\varepsilon^{X}\sim\mathcal{N}(0,I_{d})$. 
Using Jensen's inequality, we have
\begin{align*}
&\left|\mathbb{E}\left[D_t(X_{0}, Y_t, Y_0, \varepsilon^X) \,\Big| Y_{0}, Y_{t}, \varepsilon^{X}\right] - \mathbb{E}\left[D_t(X_{0}^{\mathrm{weak}}, Y_t, Y_0, \varepsilon^X) \,\Big| Y_{0}, Y_{t}, \varepsilon^{X}\right]\right| \\
&\leq \int_{\mathcal{X}} D_t(x_0, Y_t, Y_0, \varepsilon^X) | p_{X_{0}|Y_{0}}(x_{0}|Y_{0}) - p_{X_{0}|Y_{0}}^{\mathrm{weak}}(x_{0}|Y_{0})| \mathrm{d}x_{0}.
\end{align*}
It follows from Assumptions~\ref{assum:data} and~\ref{assum:network} that 
\begin{equation*}
D_t(x_0, Y_t, Y_0, \varepsilon^X) \leq K^{\prime}(1+\|Y_t\|_{2}^{\max(m,2)}+\|\varepsilon^{X}\|_{2}^{\max(m,2)}+\|Y_0\|_{2}^2), \quad x_{0}\in\mathcal{X},
\end{equation*}
where $K^{\prime}$ is a constant depending on $C$, $R_{\mathcal{X}}$, $\sigma_X, \sigma_Y$, and the variance schedule of the diffusion model. As a consequence, 
\begin{align*}
&\int_{\mathcal{X}} D_t(x_0, Y_t, Y_0, \varepsilon^X) | p_{X_{0}|Y_{0}}(x_{0}|Y_{0}) - p_{X_{0}|Y_{0}}^{\mathrm{weak}}(x_{0}|Y_{0})| \mathrm{d}x_{0} \\
&\leq 2K^{\prime}(1+\|Y_t\|_{2}^{\max(m,2)}+\|\varepsilon^{X}\|_{2}^{\max(m,2)}+\|Y_0\|_{2}^2)\|p_{X_{0}|Y_{0}}(\cdot|Y_{0})-p_{X_{0}|Y_{0}}^{\mathrm{weak}}(\cdot|Y_{0})\|_{\mathrm{TV}}.
\end{align*}
Taking expectations on both sides of the inequality implies
\begin{align}
&\left|\mathbb{E}\Bigl[D_t(X_{0}, Y_t, Y_0, \varepsilon^X)\Bigr] - \mathbb{E}\Bigl[D_t(X_{0}^{\mathrm{weak}}, Y_t, Y_0, \varepsilon^X) \Bigr]\right| \nonumber \\
&\leq \mathbb{E}\Bigl[\left|\mathbb{E}\left[D_t(X_{0}, Y_t, Y_0, \varepsilon^X) \,\Big| Y_{0}, Y_{t}, \varepsilon^{X}\right] - \mathbb{E}\left[D_t(X_{0}^{\mathrm{weak}}, Y_t, Y_0, \varepsilon^X) \,\Big| Y_{0}, Y_{t}, \varepsilon^{X}\right]\right|\Bigr] \nonumber \\
&\leq \mathbb{E}\Bigl[2K^{\prime}(1+\|Y_t\|_{2}^{\max(m,2)}+\|\varepsilon^{X}\|_{2}^{\max(m,2)}+\|Y_0\|_{2}^2)\|p_{X_{0}|Y_{0}}(\cdot|Y_{0})-p_{X_{0}|Y_{0}}^{\mathrm{weak}}(\cdot|Y_{0})\|_{\mathrm{TV}}\Bigr] \nonumber \\
&\leq 2K^{\prime}\mathbb{E}^{\frac{1}{2}}\Bigl[(1+\|Y_t\|_{2}^{\max(m,2)}+\|\varepsilon^{X}\|_{2}^{\max(m,2)}+\|Y_0\|_{2}^2)^{2}\Bigr]\mathbb{E}^{\frac{1}{2}}\Bigl[\|p_{X_{0}|Y_{0}}(\cdot|Y_{0})-p_{X_{0}|Y_{0}}^{\mathrm{weak}}(\cdot|Y_{0})\|_{\mathrm{TV}}^{2}\Bigr] \nonumber \\
&\leq M_{t}\mathbb{E}^{\frac{1}{2}}\Bigl[\|p_{X_{0}|Y_{0}}(\cdot|Y_{0})-p_{X_{0}|Y_{0}}^{\mathrm{weak}}(\cdot|Y_{0})\|_{\mathrm{TV}}\Bigr], \label{eq:proposition:delta:2}
\end{align}
where the third inequality is due to the Cauchy--Schwarz inequality, and the last inequality uses Assumption~\ref{assum:data} and the property $\|p\|_{\mathrm{TV}} \le 1$. Here $M_{t}$ is a constant depending on $C$, $R_{\mathcal{X}}$, $M_{\Xi,m}$, $d$, $\sigma_X, \sigma_Y$, and the variance schedule. Substituting~\eqref{eq:proposition:delta:2} into~\eqref{eq:proposition:delta:1} completes the proof.
\end{proof}

\subsection{Proof of Theorem~\ref{thm:error_bound}}
\label{proof:error bound}

Before presenting the main proof of the theorem, we first establish the following lemmas.

\begin{lemma}[KL divergence between Gaussians]\label{lem:kl:gaussian}
Let $\mu_{1},\mu_{2}\in\mathbb{R}^{d}$, and let $\sigma_{1},\sigma_{2}>0$. Then 
\begin{equation*}
D_{\mathrm{KL}}(\mathcal{N}(\mu_1,\sigma_1^2 I_d) \,\|\, \mathcal{N}(\mu_2, \sigma_2^2 I_d)) = \frac{d}{2}\!\left(\frac{\sigma_1^2}{\sigma_2^2} - 1 - \log\frac{\sigma_1^2}{\sigma_2^2}\right) + \frac{\|\mu_1 - \mu_2\|_2^2}{2\sigma_2^2}.
\end{equation*}
\end{lemma}

\begin{lemma}[Gaussian-smoothed KL-divergence bound]\label{lem:smoothed:kl}
Let $U\in\mathbb{R}^d$ be a random variable with a bounded density function $p_U \le M_U$ on $\mathbb{R}^{d}$ and a finite second moment $\mathbb{E}[\|U\|_{2}^2] = d\sigma_U^2$. Let $G \sim \mathcal{N}(0, \sigma_U^2 I_d)$ be an independent Gaussian random variable. Then we have the following properties:
\begin{enumerate}[label=(\roman*)]
\item The KL divergence of $p_{U}$ with respect to $p_{G}$ is finite, i.e., $D_{\mathrm{KL}}(p_{U} \,\|\, p_{G}) < \infty$.
\item For any $\sigma_{V}>0$, the KL divergence of the Gaussian-smoothed distributions satisfies the bound
\begin{equation*}
D_{\mathrm{KL}}\bigl(p_{U}\ast\gamma_{\sigma_{V}^{2}} \,\|\, p_{G}\ast\gamma_{\sigma_{V}^{2}}\bigr) \leq \left( \frac{\sigma_U^2}{\sigma_U^2 + \sigma_V^2} \right) D_{\mathrm{KL}}(p_{U} \,\|\, p_{G}),
\end{equation*}
where $\ast$ denotes the convolution, and $\gamma_{\sigma_{V}^{2}}$ represents the density of $\mathcal{N}(0,\sigma_{V}^{2}I_{d})$.
\end{enumerate}
\end{lemma}

\begin{proof}
\textbf{Part (i):} We first prove that $D_{\mathrm{KL}}(p_{U} \,\|\, p_{G}) < \infty$. By the definition of KL divergence, we decompose it into negative differential entropy and cross-entropy:
\[
D_{\mathrm{KL}}(U \,\|\, G) = \int_{\mathbb{R}^d} p_U(x) \log p_U(x) \mathrm{d}x - \int_{\mathbb{R}^d} p_U(x) \log p_G(x) \mathrm{d}x.
\]
Since $p_U\le M_U$, the negative differential entropy is strictly bounded above:
\[
\int_{\mathbb{R}^d} p_U(x) \log p_U(x) \mathrm{d}x \leq \int_{\mathbb{R}^d} p_U(x) \log(M_U) \mathrm{d}x = \log M_U.
\]
For the cross-entropy term, substituting the Gaussian density yields
\[
- \int_{\mathbb{R}^d} p_U(x) \log p_G(x) \mathrm{d}x = \frac{d}{2}\log(2\pi\sigma_U^2) + \frac{1}{2\sigma_U^2} \int_{\mathbb{R}^d} \|x\|_{2}^2 p_U(x) \mathrm{d}x = \frac{d}{2}\log(2\pi e \sigma_U^2) < \infty.
\]
Summing these two bounds guarantees $D_{\mathrm{KL}}(p_{U} \,\|\, p_{G}) \le \log M_U + \frac{d}{2} \log(2\pi e \sigma_U^2) < \infty$.

\textbf{Part (ii):} Consider the Ornstein--Uhlenbeck (OU) processes 
\begin{equation}\label{eq:lem:smoothed:kl:0}
\begin{aligned}
\mathrm{d}U_{t} &= -U_{t}\,\mathrm{d}t+\sqrt{2}\sigma_{U}\,\mathrm{d}W_{t}^{U}, \quad U_{0}\sim p_{U}, \\
\mathrm{d}G_{t} &= -G_{t}\,\mathrm{d}t+\sqrt{2}\sigma_{U}\,\mathrm{d}W_{t}^{G}, \quad G_{0}\sim p_{G},
\end{aligned}
\end{equation}
where $W_{t}^{U}$ and $W_{t}^{G}$ are $d$-dimensional independent Wiener processes. The stationary distribution of these OU processes is $\mathcal{N}(0,\sigma_{U}^{2}I_{d})$, which satisfies a log-Sobolev inequality. From~\cite{Bakry1985Diffusions}, $p_{U_{t}}$ converges to the stationary distribution $\mathcal{N}(0,\sigma_{U}^{2}I_{d})$ exponentially:
\begin{equation}\label{eq:lem:smoothed:kl:1}
D_{\mathrm{KL}}\bigl(p_{U_t} \,\|\, \gamma_{\sigma_{U}^{2}}\bigr) \leq e^{-2t} D_{\mathrm{KL}}\bigl(p_{U} \,\|\, \gamma_{\sigma_{U}^{2}}\bigr) = e^{-2t} D_{\mathrm{KL}}(p_{U} \,\|\, p_{G}).
\end{equation}
On the other hand, since $G\sim\mathcal{N}(0,\sigma_{U}^{2}I_{d})$, we have $G_{t}\sim\mathcal{N}(0,\sigma_{U}^{2}I_{d})$ for any $t>0$. Combining this with~\eqref{eq:lem:smoothed:kl:1} yields 
\begin{equation}\label{eq:lem:smoothed:kl:2}
D_{\mathrm{KL}}\bigl(p_{U_t} \,\|\, p_{G_t}\bigr) \leq e^{-2t} D_{\mathrm{KL}}(p_{U} \,\|\, p_{G}).
\end{equation}
Note that the two linear SDEs in~\eqref{eq:lem:smoothed:kl:0} admit explicit solutions
\begin{equation*}
U_t = e^{-t}U+\sqrt{1-e^{-2t}}\sigma_{U}\varepsilon_{U}, \quad G_t = e^{-t}G+\sqrt{1-e^{-2t}}\sigma_{U}\varepsilon_{G},
\end{equation*}
where $(\varepsilon_{U},\varepsilon_{G})\sim \mathcal{N}(0,I_d)\otimes\mathcal{N}(0,I_d)$ is independent of $U$ and $G$. As a result, $e^{t}U_t \sim p_{U}\ast\gamma_{\sigma_{t}^{2}}$ and $e^{t}G_t \sim p_{G}\ast\gamma_{\sigma_{t}^{2}}$, where $\sigma_{t} \coloneqq \sqrt{e^{2t}-1}\sigma_{U}$. Then for any $t>0$, since the KL divergence is invariant under the affine map, we have
\begin{equation}\label{eq:lem:smoothed:kl:3}
D_{\mathrm{KL}}\bigl(p_{U}\ast\gamma_{\sigma_{t}^{2}} \,\|\, p_{G}\ast\gamma_{\sigma_{t}^{2}}\bigr) = D_{\mathrm{KL}}\bigl(p_{e^{t}U_t} \,\|\, p_{e^{t}G_t}\bigr) = D_{\mathrm{KL}}\bigl(p_{U_t} \,\|\, p_{G_t}\bigr).
\end{equation}
By combining~\eqref{eq:lem:smoothed:kl:2} and~\eqref{eq:lem:smoothed:kl:3}, we have $D_{\mathrm{KL}}\bigl(p_{U}\ast\gamma_{\sigma_{t}^{2}} \,\|\, p_{G}\ast\gamma_{\sigma_{t}^{2}}\bigr) \leq e^{-2t} D_{\mathrm{KL}}(p_{U} \,\|\, p_{G})$. Setting $t=\frac{1}{2}\log(\frac{\sigma_U^2 + \sigma_V^2}{\sigma_U^2})$ completes the proof.
\end{proof}

\begin{proof}[Proof of Theorem~\ref{thm:error_bound}]
Recall the forward processes~\eqref{eq:forward:joint} and the normalization~\eqref{eq:normalization}:
\begin{equation}\label{eq:thm:forward}
\begin{aligned}
X_V &= \sqrt{\bar{\alpha}_V} X_0 + \sqrt{1-\bar{\alpha}_V} \varepsilon_{X}, \\
Y_S &= \sqrt{\bar{\alpha}_S} Y_0 + \sqrt{1-\bar{\alpha}_S} \varepsilon_{Y} = \sqrt{\bar{\alpha}_S} \frac{\sigma_X}{\sigma_Y} X_0 + \sqrt{\bar{\alpha}_S} \frac{1}{\sigma_Y}\Xi + \sqrt{1-\bar{\alpha}_S} \varepsilon_{Y},
\end{aligned}
\end{equation}
where $\varepsilon_X,\varepsilon_Y \sim \mathcal{N}(0, I_d)$. The proof is divided into two steps.

\noindent\textbf{Step 1. Perturbation error estimate.}
Since $X_0$ and $Y_{S}$ are conditionally independent given $Y_0$, for any $(x_{0},y_{S})\in\mathcal{X}\times\mathbb{R}^{d}$, we have
\begin{equation*}
p_{X_0|Y_S}(x_0|y_S) = \int p_{X_0|Y_0}(x_0|y_0^{\prime}) \, p_{Y_0|Y_S}(y_0^{\prime}|y_S) \,\mathrm{d}y_0^{\prime} = \mathbb{E}_{Y_0^{\prime}} \left[ p_{X_0|Y_0}(x_0|Y_0^{\prime}) \mid Y_{S}=y_{S}\right].
\end{equation*}
As a result, using the law of total expectation and Jensen's inequality, we obtain
\begin{align*}
\mathcal{E}_{\mathrm{pert}}
&\coloneqq \mathbb{E}_{Y_0, Y_S} \left[ \|p_{X_0|Y_S}(\cdot|Y_S) - p_{X_0|Y_0}(\cdot|Y_0)\|_{\mathrm{TV}} \right] \\
&\leq \mathbb{E}_{Y_S}\bigl[ \mathbb{E}_{Y_{0},Y_{0}^{\prime}}\bigl[\| p_{X_0|Y_0}(\cdot|Y_0^{\prime}) - p_{X_0|Y_0}(\cdot|Y_0) \|_{\mathrm{TV}} \mid Y_{S}\bigr]\bigr],
\end{align*}
where $Y_0$ and $Y_0'$ are conditionally independent and identically distributed given $Y_S$. Further, under Assumption~\ref{assum:compactness}, we have
\begin{equation}\label{eq:pert:1}
\mathcal{E}_{\mathrm{pert}}
\leq L_{\mathrm{post}}\mathbb{E}_{Y_S}\bigl[\mathbb{E}_{Y_{0},Y_{0}^{\prime}}\bigl[\|Y_0^{\prime}-Y_0\|_{2} \mid Y_{S}\bigr]\bigr],
\end{equation}
where $Y_0, Y_0'$ are conditionally independent and identically distributed given $Y_S$. Then
\begin{align}
\mathbb{E}_{Y_{0},Y_{0}^{\prime}}\bigl[\|Y_0^{\prime}-Y_0\|_{2}^{2} \mid Y_{S}\bigr]
&=\mathbb{E}_{Y_{0},Y_{0}^{\prime}}\bigl[\|Y_0^{\prime}-\mathbb{E}[Y_0^{\prime}\mid Y_{S}]+\mathbb{E}[Y_0\mid Y_{S}]-Y_0\|_{2}^{2} \mid Y_{S}\bigr] \nonumber \\
&=\mathbb{E}_{Y_{0}^{\prime}}\bigl[\|Y_0^{\prime}-\mathbb{E}[Y_0^{\prime}\mid Y_{S}]\|_{2}^{2} \mid Y_{S}\bigr]+\mathbb{E}_{Y_{0}}\bigl[\|\mathbb{E}[Y_0\mid Y_{S}]-Y_0\|_{2}^{2} \mid Y_{S}\bigr] \nonumber \\
&=2\mathbb{E}_{Y_{0}}\bigl[\|\mathbb{E}[Y_0\mid Y_{S}]-Y_0\|_{2}^{2} \mid Y_{S}\bigr]. \label{eq:pert:2}
\end{align}
For any $z\in\mathbb{R}^{d}$, we have
\begin{align*}
\mathbb{E}_{Y_{0}}\bigl[\|z-Y_0\|_{2}^{2} \mid Y_{S}\bigr]
&=\mathbb{E}_{Y_{0}}\bigl[\|z-\mathbb{E}[Y_0\mid Y_{S}]+\mathbb{E}[Y_0\mid Y_{S}]-Y_0\|_{2}^{2} \mid Y_{S}\bigr] \\
&=\mathbb{E}_{Y_{0}}\bigl[\|z-\mathbb{E}[Y_0\mid Y_{S}]\|_{2}^{2} \mid Y_{S}\bigr]+\mathbb{E}_{Y_{0}}\bigl[\|\mathbb{E}[Y_0\mid Y_{S}]-Y_0\|_{2}^{2} \mid Y_{S}\bigr] \\
&\quad +2\mathbb{E}_{Y_{0}}\bigl[\langle z-\mathbb{E}[Y_0\mid Y_{S}],\mathbb{E}[Y_0\mid Y_{S}]-Y_0\rangle \mid Y_{S}\bigr] \\
&=\mathbb{E}_{Y_{0}}\bigl[\|z-\mathbb{E}[Y_0\mid Y_{S}]\|_{2}^{2} \mid Y_{S}\bigr]+\mathbb{E}_{Y_{0}}\bigl[\|\mathbb{E}[Y_0\mid Y_{S}]-Y_0\|_{2}^{2} \mid Y_{S}\bigr] \\
&\geq \mathbb{E}_{Y_{0}}\bigl[\|\mathbb{E}[Y_0\mid Y_{S}]-Y_0\|_{2}^{2} \mid Y_{S}\bigr].
\end{align*}
Setting $z=\frac{1}{\sqrt{\bar{\alpha}_S}} Y_S$ in this inequality and using~\eqref{eq:thm:forward}, we have  
\begin{equation}\label{eq:pert:3}
\mathbb{E}_{Y_{0}}\bigl[\|\mathbb{E}[Y_0\mid Y_{S}]-Y_0\|_{2}^{2} \mid Y_{S}\bigr] \leq \mathbb{E}_{Y_{0}}\Bigl[\|\sqrt{\frac{1-\bar{\alpha}_S}{\bar{\alpha}_S}} \varepsilon_{Y}\|_{2}^{2} \,\Big|\, Y_{S}\Bigr],
\end{equation}
where $\varepsilon_{Y}\sim\mathcal{N}(0,I_{d})$ is independent of $Y_{0}$. Combining~\eqref{eq:pert:2} and~\eqref{eq:pert:3} and taking the expectation with respect to $Y_{S}$ yields
\begin{equation}\label{eq:pert:4}
\mathbb{E}_{Y_S}\bigl[\mathbb{E}_{Y_{0},Y_{0}^{\prime}}\bigl[\|Y_0^{\prime}-Y_0\|_{2}^{2} \mid Y_{S}\bigr]\bigr] \leq 2\mathbb{E}_{\varepsilon_{Y}\sim\mathcal{N}(0,I_{d})}\Bigl[\|\sqrt{\frac{1-\bar{\alpha}_S}{\bar{\alpha}_S}} \varepsilon_{Y}\|_{2}^{2}\Bigr] = 2d\frac{1-\bar{\alpha}_S}{\bar{\alpha}_S}.
\end{equation}
Substituting~\eqref{eq:pert:4} into~\eqref{eq:pert:1} and using Jensen's inequality, we obtain 
\begin{equation}\label{eq:pert:5}
\mathcal{E}_{\mathrm{pert}}
\leq L_{\mathrm{post}}\sqrt{\mathbb{E}_{Y_S}\bigl[\mathbb{E}_{Y_{0},Y_{0}^{\prime}}\bigl[\|Y_0^{\prime}-Y_0\|_{2}^{2} \mid Y_{S}\bigr]\bigr]} \leq L_{\mathrm{post}}\sqrt{2d\frac{1-\bar{\alpha}_S}{\bar{\alpha}_S}}.
\end{equation}

\noindent\textbf{Step 2. Alignment error estimate.}
Applying the triangle inequality, the alignment error can be decomposed as
\begin{align*}
\mathcal{E}_{\mathrm{align}}
&\coloneqq \mathbb{E}_{Y_S} \left[ \|p_{X_0|X_V}(\cdot|Y_S) - p_{X_0|Y_S}(\cdot|Y_S)\|_{\mathrm{TV}} \right] \\
&= \frac{1}{2}\iint |p_{X_0|X_V}(x_{0}|y_{S})-p_{X_0|Y_S}(x_{0}|y_{S})| \, p_{Y_{S}}(y_{S}) \,\mathrm{d}x_{0}\,\mathrm{d}y_{S} \\
&\leq \frac{1}{2}\iint p_{X_0|X_V}(x_{0}|y_{S}) |p_{Y_{S}}(y_{S})-p_{X_{V}}(y_{S})| \,\mathrm{d}x_{0}\,\mathrm{d}y_{S} \\
&\quad +\frac{1}{2}\iint |p_{X_0|X_V}(x_{0}|y_{S})p_{X_{V}}(y_{S})-p_{X_0|Y_S}(x_{0}|y_{S})p_{Y_{S}}(y_{S})| \,\mathrm{d}x_{0}\,\mathrm{d}y_{S} \\
&= \|p_{Y_{S}}-p_{X_{V}}\|_{\mathrm{TV}}+\mathbb{E}_{X_{0}}\bigl[\|p_{X_V|X_0}(\cdot|X_{0})-p_{Y_S|X_{0}}(\cdot|X_{0})\|_{\mathrm{TV}}\bigr],
\end{align*}
where the last equality is due to Bayes' rule. For the first summand, it follows from Jensen's inequality that 
\begin{align*}
\|p_{Y_{S}}-p_{X_{V}}\|_{\mathrm{TV}}
&=\|\int p_{Y_{S}|X_{0}}(\cdot|x_{0})p_{X_{0}}(x_{0})\,\mathrm{d}x_{0}-\int p_{X_{V}|X_{0}}(\cdot|x_{0})p_{X_{0}}(x_{0})\,\mathrm{d}x_{0}\|_{\mathrm{TV}} \\
&\leq \mathbb{E}_{X_{0}}\bigl[\|p_{X_V|X_0}(\cdot|X_{0})-p_{Y_S|X_{0}}(\cdot|X_{0})\|_{\mathrm{TV}}\bigr].
\end{align*}
As a consequence, we have 
\begin{equation}\label{eq:align:1}
\mathcal{E}_{\mathrm{align}} \leq 2\mathbb{E}_{X_{0}}\bigl[\|p_{X_V|X_0}(\cdot|X_{0})-p_{Y_S|X_{0}}(\cdot|X_{0})\|_{\mathrm{TV}}\bigr].
\end{equation}
From~\eqref{eq:thm:forward}, we have 
\begin{align*}
(X_{V} \mid X_{0}=x_{0}) &\stackrel{\mathrm{d}}{=} \sqrt{\bar{\alpha}_V} x_0 + \sqrt{1-\bar{\alpha}_V} \varepsilon_{X}, \\
(Y_S \mid X_{0}=x_{0}) &\stackrel{\mathrm{d}}{=} \sqrt{\bar{\alpha}_S} \frac{\sigma_X}{\sigma_Y} x_0 + \sqrt{\bar{\alpha}_S} \frac{1}{\sigma_Y}\Xi + \sqrt{1-\bar{\alpha}_S} \varepsilon_{Y}.
\end{align*}
Then we construct two auxiliary random variables 
\begin{align*}
(E \mid X_{0}=x_{0}) &\stackrel{\mathrm{d}}{=} \sqrt{\bar{\alpha}_S} \frac{\sigma_X}{\sigma_Y} x_0 + \sqrt{1-\bar{\alpha}_V} \varepsilon_{X}, \\
(F \mid X_{0}=x_{0}) &\stackrel{\mathrm{d}}{=} \sqrt{\bar{\alpha}_S} \frac{\sigma_X}{\sigma_Y} x_0 + \sqrt{\bar{\alpha}_S} \frac{1}{\sigma_Y}\varepsilon_{\Xi} + \sqrt{1-\bar{\alpha}_S} \varepsilon_{Y}, 
\end{align*}
where $\varepsilon_{\Xi}\sim\mathcal{N}(0,\sigma_{\Xi}^{2}I_{d})$. Applying the triangle inequality gives, for any $x_{0}\in\mathcal{X}$,
\begin{align}
&\|p_{X_V|X_0}(\cdot|x_{0})-p_{Y_S|X_{0}}(\cdot|x_{0})\|_{\mathrm{TV}} \nonumber \\
&\leq \underbrace{\|p_{X_V|X_0}(\cdot|x_{0})-p_{E|X_0}(\cdot|x_{0})\|_{\mathrm{TV}}}_{\text{mean shift}}+\underbrace{\|p_{E|X_0}(\cdot|x_{0})-p_{F|X_0}(\cdot|x_{0})\|_{\mathrm{TV}}}_{\text{variance mismatch}} \nonumber \\
&\quad +\underbrace{\|p_{F|X_0}(\cdot|x_{0})-p_{Y_S|X_{0}}(\cdot|x_{0})\|_{\mathrm{TV}}}_{\text{error of Gaussian approximation}}, \label{eq:align:2}
\end{align}
For the first summand in~\eqref{eq:align:2}, using Pinsker's inequality and Lemma~\ref{lem:kl:gaussian}, we have 
\begin{align}
\|p_{X_V|X_0}(\cdot|x_{0})-p_{E|X_0}(\cdot|x_{0})\|_{\mathrm{TV}} 
&\leq \sqrt{\frac{1}{2}D_{\mathrm{KL}}(p_{X_V|X_0}(\cdot|x_{0}) \| p_{E|X_0}(\cdot|x_{0}))} \nonumber \\
&\leq \frac{|\sigma_Y \sqrt{\bar{\alpha}_V} - \sigma_X \sqrt{\bar{\alpha}_S}| \|x_{0}\|_2}{2\sigma_Y \sqrt{1-\bar{\alpha}_V}}. \label{eq:align:3}
\end{align}
For the second summand in~\eqref{eq:align:2}, Pinsker's inequality and Lemma~\ref{lem:kl:gaussian} imply 
\begin{align}
\|p_{E|X_0}(\cdot|x_{0})-p_{F|X_0}(\cdot|x_{0})\|_{\mathrm{TV}}
&\leq \sqrt{\frac{1}{2} D_{\mathrm{KL}}(p_{E|X_0}(\cdot|x_{0}) \,\|\, p_{F|X_0}(\cdot|x_{0}))} \nonumber \\
&\leq  \sqrt{\frac{d}{2}} \frac{|\sigma_X^2\bar{\alpha}_S -\sigma_Y^2\bar{\alpha}_V|}{\sigma_Y \sqrt{1-\bar{\alpha}_V} \sqrt{\sigma_Y^2 - \bar{\alpha}_S \sigma_X^2}} \nonumber \\
&\leq  \sqrt{\frac{d}{2}} \frac{(\sigma_{X}+\sigma_{Y})|\sigma_X\sqrt{\bar{\alpha}_S}-\sigma_Y\sqrt{\bar{\alpha}_V}|}{\sigma_Y \sqrt{1-\bar{\alpha}_V} \sqrt{\sigma_Y^2 - \bar{\alpha}_S \sigma_X^2}} \nonumber \\
&\leq \sqrt{\frac{2d}{\sigma_Y^2 - \bar{\alpha}_S \sigma_X^2}}\frac{|\sigma_X\sqrt{\bar{\alpha}_S}-\sigma_Y\sqrt{\bar{\alpha}_V}|}{\sqrt{1-\bar{\alpha}_V}}, \label{eq:align:4}
\end{align}
For the third summand in~\eqref{eq:align:2}, it follows from Lemma~\ref{lem:smoothed:kl} that 
\begin{align}
\|p_{F|X_0}(\cdot|x_{0})-p_{Y_S|X_{0}}(\cdot|x_{0})\|_{\mathrm{TV}}
&\leq \sqrt{\frac{1}{2} D_{\mathrm{KL}}( p_{Y_S|X_{0}}(\cdot|x_{0})\,\|\, p_{F|X_0}(\cdot|x_{0}))} \nonumber \\
&\leq \sqrt{\frac{\bar{\alpha}_S \sigma_{\Xi}^2}{2(\sigma_Y^2 - \bar{\alpha}_S \sigma_X^2)}} \sqrt{D_{\mathrm{KL}}(p_{\Xi} \,\|\, \mathcal{N}(0, \sigma_{\Xi}^2 I_d))}. \label{eq:align:5}
\end{align}
Substituting~\eqref{eq:align:3},~\eqref{eq:align:4} and~\eqref{eq:align:5} into~\eqref{eq:align:2} yields, for any $x_{0}\in\mathcal{X}$, 
\begin{align*}
&\|p_{X_V|X_0}(\cdot|x_{0})-p_{Y_S|X_{0}}(\cdot|x_{0})\|_{\mathrm{TV}} \\
&\leq \Biggl(\frac{R_{\mathcal{X}}}{2\sigma_Y}+\sqrt{\frac{2d}{\sigma_Y^2 - \bar{\alpha}_S \sigma_X^2}}\Biggr)\frac{|\sigma_Y \sqrt{\bar{\alpha}_V} - \sigma_X \sqrt{\bar{\alpha}_S}|}{\sqrt{1-\bar{\alpha}_V}}  \\
&\quad +\sqrt{\frac{\bar{\alpha}_S \sigma_{\Xi}^2}{2(\sigma_Y^2 - \bar{\alpha}_S \sigma_X^2)}} \sqrt{D_{\mathrm{KL}}(p_{\Xi} \,\|\, \mathcal{N}(0, \sigma_{\Xi}^2 I_d))}.
\end{align*}
Substituting this into~\eqref{eq:align:1} implies
\begin{align}
\mathcal{E}_{\mathrm{align}}
&\leq \Biggl(\frac{R_{\mathcal{X}}}{\sigma_Y}+\sqrt{\frac{8d}{\sigma_Y^2 - \bar{\alpha}_S \sigma_X^2}}\Biggr)\frac{|\sigma_Y \sqrt{\bar{\alpha}_V} - \sigma_X \sqrt{\bar{\alpha}_S}|}{\sqrt{1-\bar{\alpha}_V}} \nonumber \\
&\quad +\sqrt{\frac{2\bar{\alpha}_S \sigma_{\Xi}^2}{\sigma_Y^2 - \bar{\alpha}_S \sigma_X^2}} \sqrt{D_{\mathrm{KL}}(p_{\Xi} \,\|\, \mathcal{N}(0, \sigma_{\Xi}^2 I_d))}. \label{eq:align:6}
\end{align}
Finally, combining~\eqref{eq:error:decomp},~\eqref{eq:pert:5} and~\eqref{eq:align:6} completes the proof.
\end{proof}

\subsection{Proof of Corollary~\ref{cor:conditional_generative_error}}
\label{proof:error:condition}

\begin{proof}[Proof of Corollary~\ref{cor:conditional_generative_error}]
By the same argument used to obtain~\eqref{eq:align:1}, we have
\begin{equation*}
\mathbb{E}_{X_{0}}\bigl[\|p_{Y_0|X_0}(\cdot|X_0)-p_{Y_0|X_0}^{\mathrm{weak}}(\cdot|X_0)\|_{\mathrm{TV}}\bigr] \leq 2\mathbb{E}_{Y_{0}}\bigl[\|p_{X_0|Y_0}(\cdot|Y_0)-p_{X_0|Y_0}^{\mathrm{weak}}(\cdot|Y_0)\|_{\mathrm{TV}}\bigr].
\end{equation*}
Combining this with Theorem~\ref{thm:error_bound} completes the proof.
\end{proof}

\subsection{Extension to signal-dependent noise}
\label{supp:sec:conditional-analysis}

The analysis in Section~\ref{sec:analysis} assumes that the additive noise is
independent of the clean signal. Here we summarize how the same argument
extends to signal-dependent noise. Consider the normalized observation model
\begin{equation}
    Y_0=\frac{\sigma_X}{\sigma_Y}X_0
    +\frac{1}{\sigma_Y}\Xi(X_0),
    \label{eq:conditional-noise-model}
\end{equation}
where the conditional law of \(\Xi(X_0)\) may depend on \(X_0\). We assume
that \(X_0\) is supported on the compact set \(\mathcal X\), and, for every
\(x\in\mathcal X\),
\[
    \mathbb E[\Xi(X_0)\mid X_0=x]=0,
    \qquad
    \operatorname{Cov}(\Xi(X_0)\mid X_0=x)=\Sigma_x.
\]
The conditional densities are uniformly bounded, the traces
\(\operatorname{Tr}(\Sigma_x)\) and the moments required by
Assumption~\ref{assum:network} are uniformly bounded, and the conditional
noise field is independent of the Gaussian diffusion noises. Define
\[
    \sigma_\Xi^2
    \coloneqq \frac{1}{d}\mathbb E_{X_0}[\operatorname{Tr}(\Sigma_{X_0})].
\]
Conditional centering then gives the averaged variance identity
\(\sigma_X^2+\sigma_\Xi^2=\sigma_Y^2\).

Let
\[
    \mathfrak D_{\Xi\mid X}
    \coloneqq
    \mathbb E_{X_0}\!\left[
      \sqrt{D_{\mathrm{KL}}\!\left(
      p_{\Xi\mid X_0}(\cdot\mid X_0)
      \,\big\|\,
      \mathcal N(0,\sigma_\Xi^2 I_d)
      \right)}
    \right].
\]
This quantity captures anisotropy, spatially varying conditional variance,
and higher-order non-Gaussian structure through a single trace-matched
Gaussian reference. Assume additionally that the posterior map is
\(L_{\mathrm{post}}\)-Lipschitz in total variation, as in
Assumption~\ref{assum:compactness}.

\begin{proposition}[Conditional-noise extension]
\label{prop:conditional-noise-extension}
Under the conditions above and Assumption~\ref{assum:network}, the loss gap is
controlled by the posterior mismatch exactly as in
Proposition~\ref{proposition:delta}. If the diffusion times satisfy the
trace-averaged alignment condition
\[
    \bar\alpha_V\sigma_Y^2=\bar\alpha_S\sigma_X^2,
\]
then
\begin{equation}
\begin{split}
    \Delta_p
    \le{}&
    L_{\mathrm{post}}
    \sqrt{\frac{2d(1-\bar\alpha_S)}{\bar\alpha_S}}
    \\
    &+\sqrt{2}
    \sqrt{\frac{\bar\alpha_S\sigma_\Xi^2}
    {\sigma_Y^2-\bar\alpha_S\sigma_X^2}}
    \,\mathfrak D_{\Xi\mid X}.
\end{split}
\label{eq:conditional-noise-bound}
\end{equation}
Consequently,
\[
    \Delta
    \le
    \sum_{t=1}^T\frac{M_t}{\sigma_t^2}\sqrt{\Delta_p},
\]
for finite constants \(M_t\) depending on the diffusion schedule, the network
growth bound, and the stated moment bounds.
\end{proposition}

\begin{proof}[Proof sketch]
Conditioning on \(X_0=x\), compare the forward laws of \(X_V\) and \(Y_S\)
with a Gaussian having the global trace-matched variance. Pinsker's inequality
splits their total-variation discrepancy into mean, variance, and conditional
non-Gaussianity terms. The alignment relation cancels the first two terms,
leaving the second term in~\eqref{eq:conditional-noise-bound}. The posterior
Lipschitz property bounds the information loss incurred by diffusing \(Y_0\)
to \(Y_S\), which gives the first term. The loss-gap estimate then follows by
the same conditional-expectation and Cauchy--Schwarz argument used in
Proposition~\ref{proposition:delta}.
\end{proof}

For completeness, the posterior stability assumption has a simple sufficient
condition in the independent-noise special case. If \(p_\Xi>0\),
\(\log p_\Xi\in C^2(\mathbb R^d)\), and
\(\sup_\xi\|\nabla^2\log p_\Xi(\xi)\|_{\mathrm{op}}\le H_\Xi\), then
\begin{equation}
    \left\|p_{X_0\mid Y_0}(\cdot\mid y_1)
    -p_{X_0\mid Y_0}(\cdot\mid y_2)\right\|_{\mathrm{TV}}
    \le
    \frac{\sigma_X\sigma_Y H_\Xi}{4}
    \operatorname{diam}(\mathcal X)\,\|y_1-y_2\|_2.
    \label{eq:posterior-stability-sufficient}
\end{equation}
For Gaussian noise \(\mathcal N(0,\Sigma)\), one may take
\(H_\Xi=\|\Sigma^{-1}\|_{\mathrm{op}}\).

\section{Additional Experiments}\label{supp:sec:extra-experiments}

\subsection{Independence Test}
\label{supp:subsec:independence-test}
\begin{figure}[t]
    \centering
    \includegraphics[width=1\textwidth, keepaspectratio]{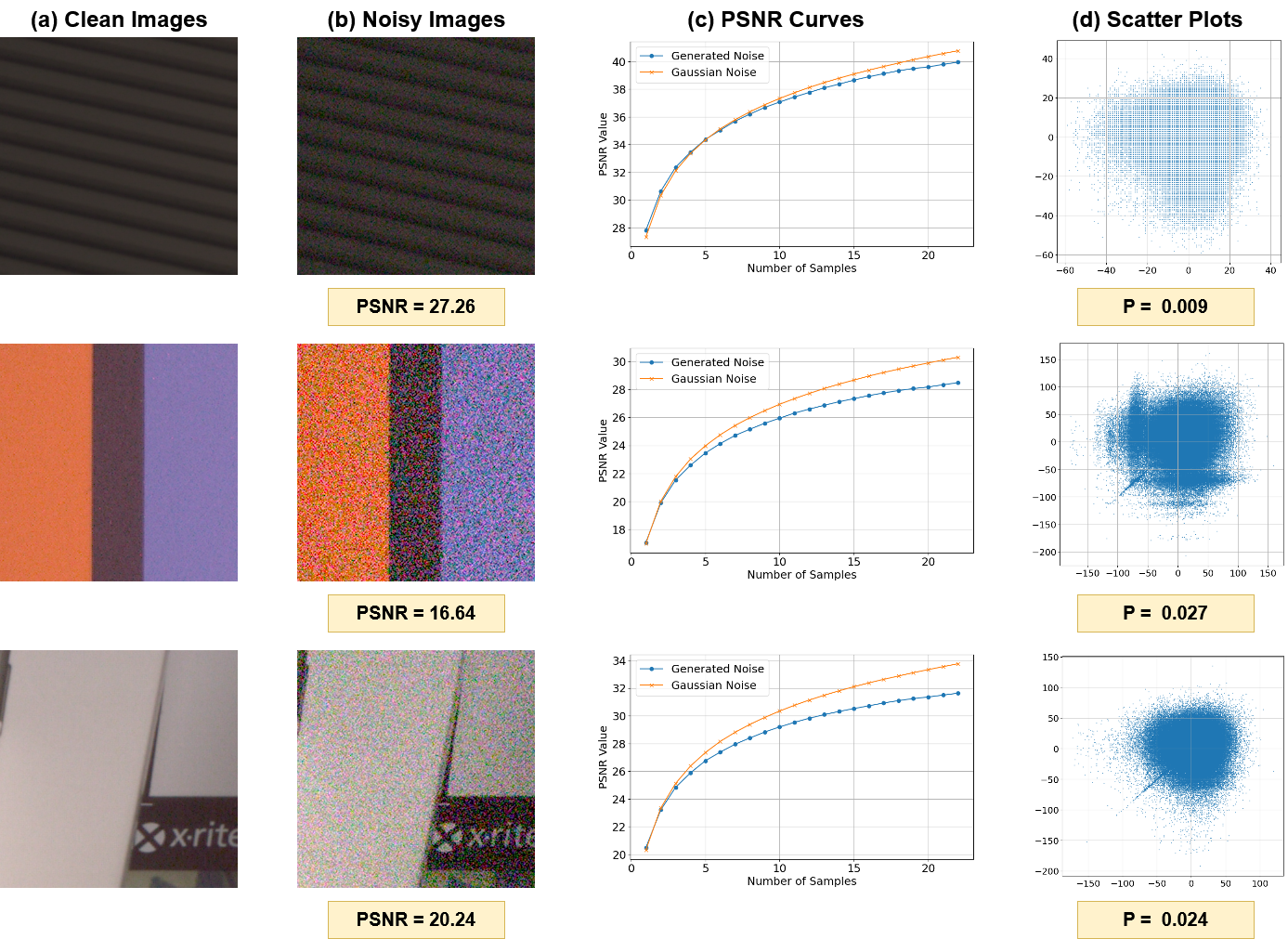}
    \caption{Independence test of the generated noise. (c) shows the change in PSNR of the averaged image as the number of noise samples increases. The curve closely follows the PSNR variation of simulated independent Gaussian noise, asymptotically approaching a theoretical logarithmic curve. (d) presents a scatter plot of the noise distribution between two generated noise images. The approximately circular shape of the distribution indicates that the two noise images exhibit weak correlation. The displayed \(P\) values in (d) denote Pearson correlation coefficients.}
    \label{supp:fig:independence-test}
\end{figure}
To verify that our model can generate random noise rather than learning a fixed noise pattern for each image, we generate multiple noise samples for the same image on the SIDD validation set and analyze the independence among these noise samples. First, according to the law of large numbers, if the noise is independent, the average of multiple noise images should converge to the clean image. We compute the change in PSNR between the averaged image and the clean image as the number of noise samples used for averaging increases. For comparison, we also average independent Gaussian noise added to the clean image. The results are shown in Figure~\ref{supp:fig:independence-test}(c). It can be observed that the PSNR curve of the averaged generated noise closely aligns with that of the independent Gaussian noise, indicating strong independence in the generated noise.

Secondly, we randomly select two generated noise images and plot a scatter diagram of their pixel values, as shown in Figure~\ref{supp:fig:independence-test}(d). If the noise were highly correlated, the scatter plot would exhibit a distinct diagonal distribution; whereas if the noise is independent, the scatter plot should approximate a circular distribution. As shown in the figure, the scatter plot is nearly circular, with a few strongly correlated regions observed in high-noise images. This may be attributed to the increased difficulty in distinguishing between signal and noise under high-intensity noise conditions, leading the model to misinterpret certain noise as signal and thereby introducing some correlation. We also calculate the Pearson correlation coefficient between these two noise samples, which is consistently below 0.05, further confirming that the generated noise samples are independent of each other.

\subsection{Training and Sampling Hyperparameters}
\label{supp:subsec:hyperparameter-selection}
\begin{figure}[t]
    \centering
    \includegraphics[width=1\textwidth, keepaspectratio]{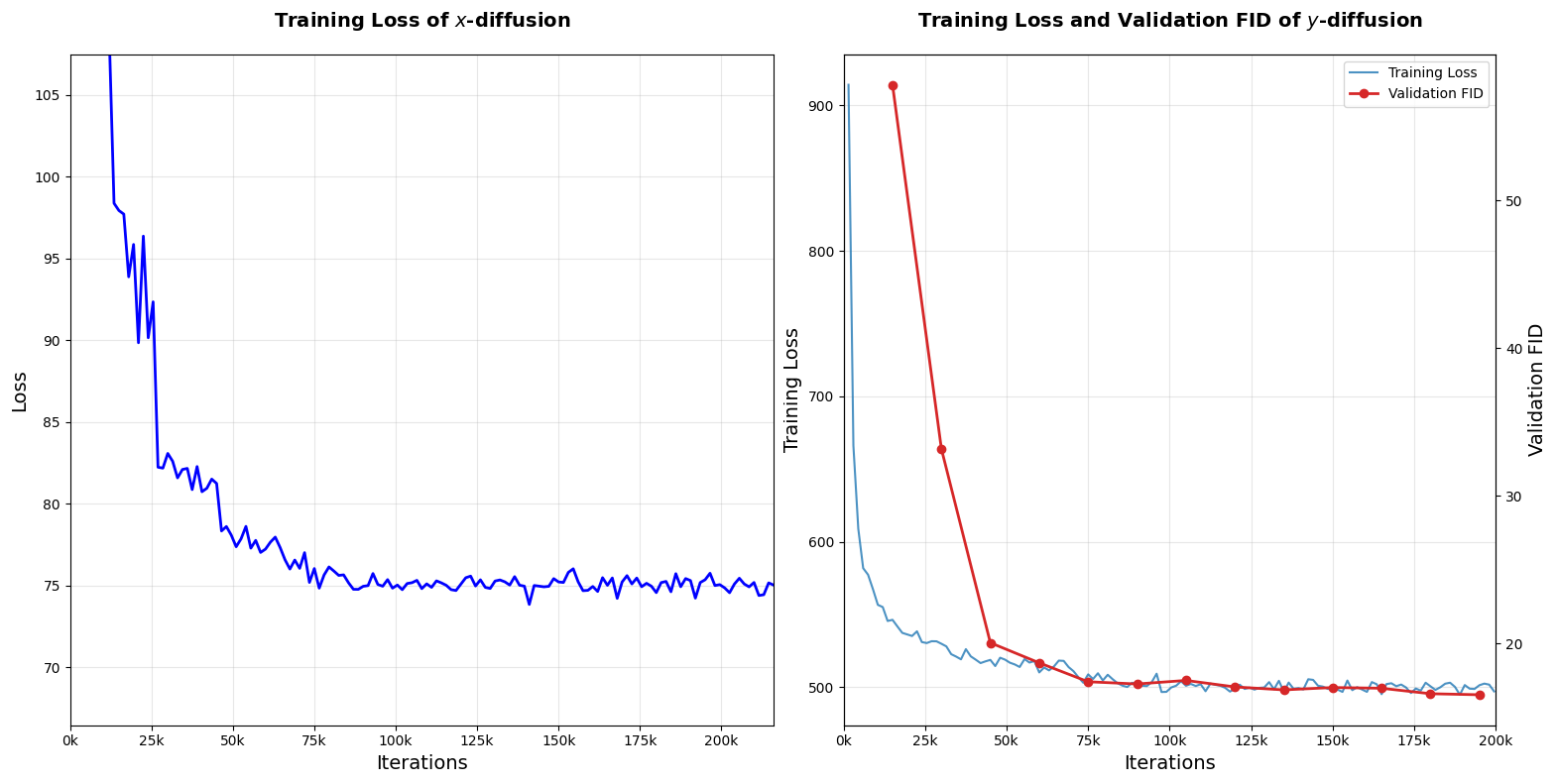}
    \caption{Curves of loss and FID versus the number of iteration steps. (a) shows the decreasing loss curve for the diffusion network regarding $x$; (b) shows the loss and FID variation curves for the conditional diffusion network regarding $y$.}
    \label{supp:fig:step-selection}
\end{figure}
To select the optimal number of iteration steps for training the two diffusion networks, we recorded the variation curves of their losses with increasing iteration steps on the WUP-SIDD data. For the conditional diffusion network, we also tracked the FID variation curve for validation set sampling at every 50 steps. The results are shown in Figure~\ref{supp:fig:step-selection}. It can be observed that the losses of both models stop decreasing after 100k iterations, and the FID on the validation set for the conditional diffusion model no longer shows significant improvement. Therefore, 100k iteration steps already provide a stable validation trend in this setting. In the main experiments, we use task-dependent training budgets: 100k iterations for simulated image noise and a more conservative 300k iterations for SIDD, as reported in the main text.

To determine the optimal sampling steps and the random term $\eta$, we evaluated the generation quality on the SIDD validation set across different sampling steps and different values of $\eta$. The results are shown in Figure~\ref{supp:fig:sampling-selection}. From the line graph, it can be seen that selecting 200 sampling steps with $\eta=0.5$ yields the best generation quality. However, when sampling from the training set to generate simulated paired data, we prioritize faster sampling while preserving stochasticity; in the SIDD experiments, we use 50 sampling steps and set $\eta=1.0$, as stated in the main text.
\begin{figure}[t]
    \centering
    \includegraphics[width=0.8\textwidth, keepaspectratio]{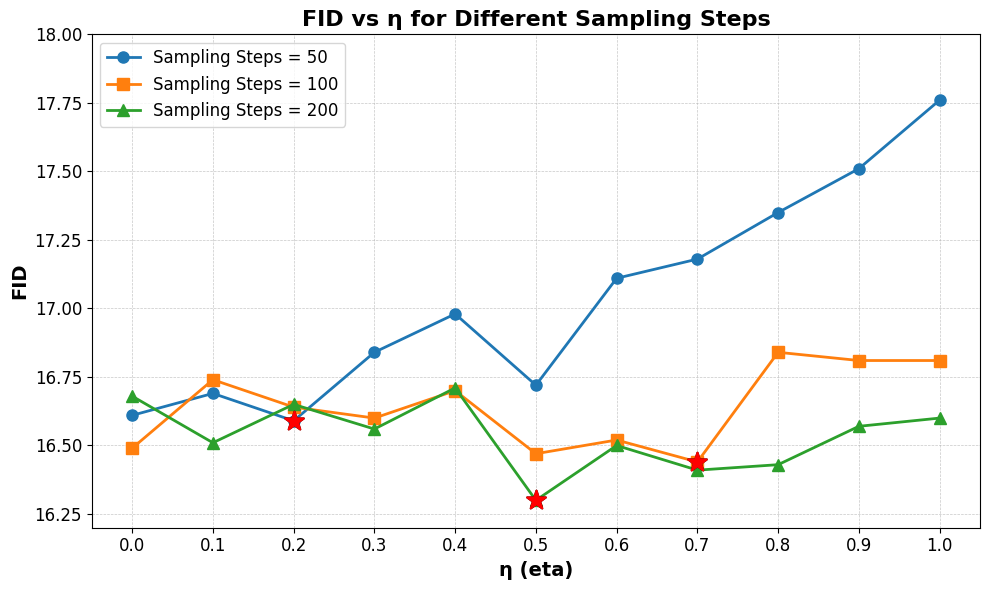}
    \caption{FID variation curves with respect to the number of sampling steps and the random term $\eta$.}
    \label{supp:fig:sampling-selection}
\end{figure}

\subsection{Analysis of Degradation Level}
\label{supp:subsec:degradation-level}
\begin{figure}[t]
    \centering
    \includegraphics[width=1\textwidth, keepaspectratio]{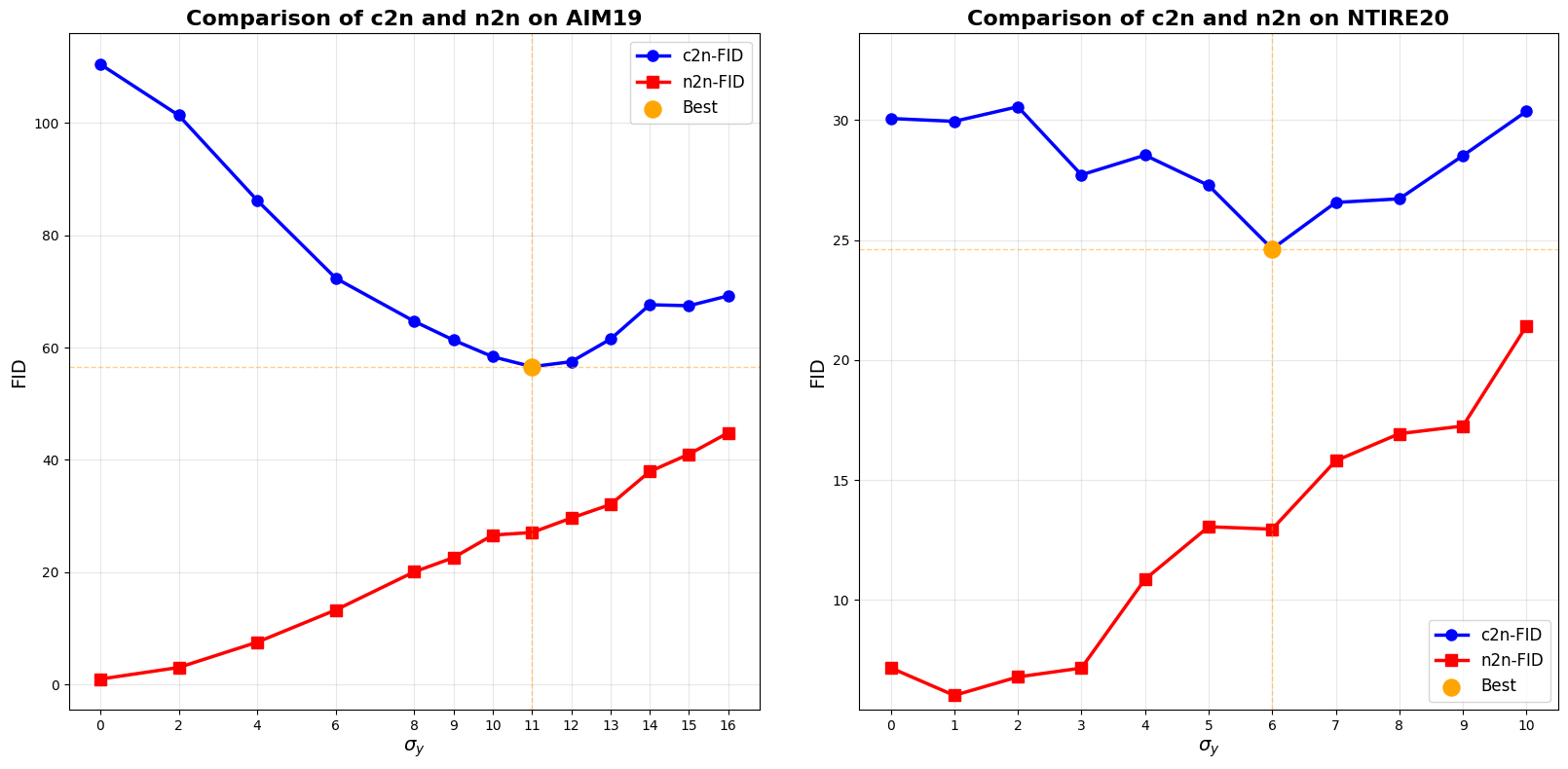}
    \caption{Impact of different $\tau_S$ values on the generated results for the AIM19 and NTIRE20 datasets.}
    \label{supp:fig:degradation-level}
\end{figure}
We investigate the influence of different effective diffusion noise levels $\tau_S$ on noise generation outcomes using the AIM19 and NTIRE20 datasets. For $\tau_V$, we consistently set
\begin{align*}
    \tau_V^2 = \tau_S^2 + \hat{\sigma}_n^2,
\end{align*}
where $\hat{\sigma}_n$ denotes the average noise standard deviation in the dataset.
In this ablation, the labels ``c2n'' and ``n2n'' indicate the source of the condition used by the conditional diffusion model. The ``c2n'' curve uses a clean image as the condition to generate a noisy image, whereas ``n2n'' first obtains a weak reconstruction from a noisy image and then uses this reconstruction as the condition to generate a noisy image.
The results are shown in Figure~\ref{supp:fig:degradation-level}. It can be observed that on both datasets, the FID curve for ``n2n'' is nearly monotonically increasing. The ``n2n'' condition is already close to the weak-reconstruction condition used during training when \(\tau_S\) is small. Increasing \(\tau_S\) injects additional Gaussian perturbations into this condition and therefore disrupts structural details, leading to worse FID. The ``c2n'' curves initially decline rapidly and then rise slowly. The decline corresponds to improved alignment between the corrupted clean and noisy states, while the subsequent rise reflects the increasing information loss caused by excessive perturbation. This illustrates the weak-coupling trade-off described in Section~\ref{sec:analysis}. Furthermore, note that the $\tau_S$ required for the FID to reach its optimal value on AIM19 is significantly larger than that on NTIRE20, which is qualitatively consistent with the dataset-dependent noise statistics in the two test sets.

\subsection{Optimal Denoising Steps}
\label{supp:subsec:denoising-steps}
Because the diffusion model for $x$ is applied only from the aligned state $x_{V}$, the reconstruction distribution $p_{X_0|X_V}$ is sharply concentrated around the clean image for datasets with relatively small noise variance, such as AIM19 and NTIRE20. Therefore, we choose one-step sampling to improve the efficiency of subsequent training for the conditional diffusion model.

For the SIDD dataset, where the noise variance is larger, we investigated the impact of different sampling steps on both the weak reconstruction \(p_{X_0|X_V}(\cdot|Y_S)\) and the final generation results.
The results are shown in Figure~\ref{supp:fig:denoising-steps}. It can be observed that the final generation FID improves rapidly within the first few steps and reaches a near-optimal range around 4--6 steps.
Using more steps does not further improve generation quality and may even degrade it. We therefore set the number of sampling steps to 5 as an efficiency-quality trade-off.
\begin{figure}[H]
    \centering
    \includegraphics[width=1\textwidth, keepaspectratio]{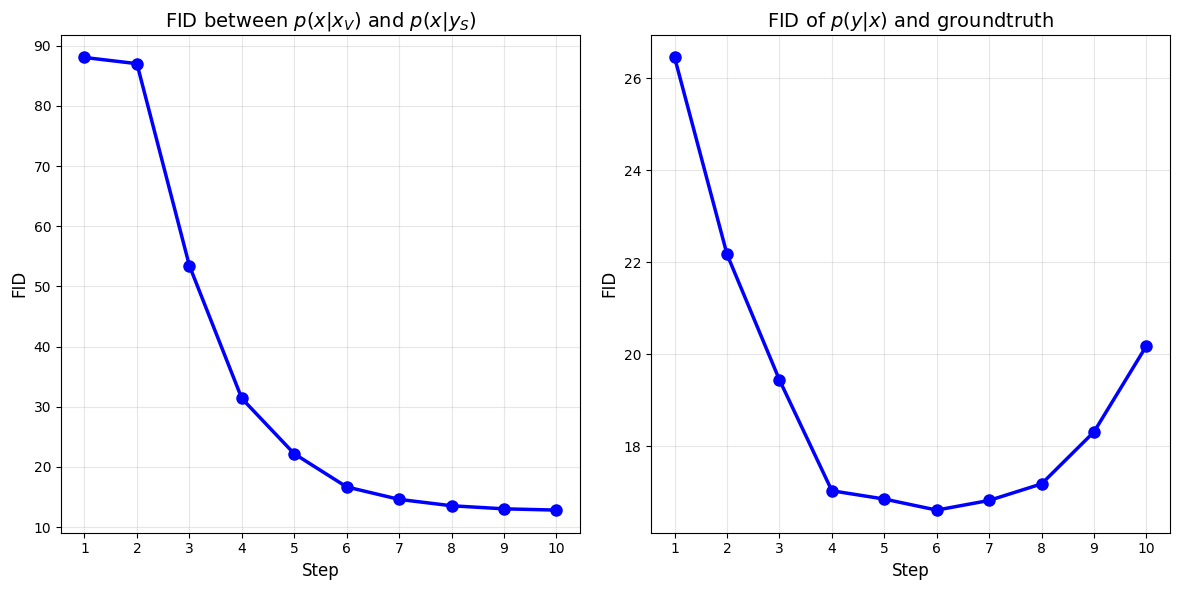}
    \caption{Impact of different sampling steps on the generation results for the SIDD dataset.}
    \label{supp:fig:denoising-steps}
\end{figure}

\clearpage
\subsection{Role of the Two-Path Diffusion Architecture}
\label{supp:subsec:two-path-noise-modeling}
\label{supp:subsec:noisy-data-usage}
Note that the clean-domain diffusion model in LUD-DIF can be used by itself to map a noisy observation to a clean reconstruction. Given a noisy observation \(y_0\), one may first diffuse it to \(y_S\), identify the aligned state with \(x_V\) through the weak-coupling approximation \(Y_S \equiv X_V\), and then apply the reverse clean-domain process \(p_{X_0|X_V}(\cdot|y_S)\) to obtain \(x_0\). This direct weak-reconstruction baseline can be summarized as \(y_0 \to y_S \equiv x_V \to x_0\), and we refer to it as the noise-to-clean (N2C) baseline in Table~\ref{supp:tab:c2n-n2c}.

In contrast, the full two-path diffusion architecture uses both the clean-domain diffusion model and the learned conditional diffusion model for \(Y|X\). In the clean-to-noisy (C2N) generation mode, we synthesize noisy observations from clean images, and the generated clean--noisy pairs are used to train the same downstream DnCNN denoiser. The comparison in Table~\ref{supp:tab:c2n-n2c} evaluates whether explicitly modeling the noisy-domain conditional path provides useful information beyond direct weak reconstruction. The results show that the two-path C2N scheme achieves better downstream denoising performance.
\begin{table}[H]
\centering
\caption{Effect of the Two-Path Diffusion Architecture for Noise Modeling}
\label{supp:tab:c2n-n2c}
\begin{tabular}{c|cc}
\hline
Method & PSNR $\uparrow$ & SSIM $\uparrow$ \\
\hline
Direct weak reconstruction (N2C) & 32.17 & 0.847 \\
\textbf{Two-path C2N (Ours)} & \textbf{35.46} & \textbf{0.896} \\
\hline
\end{tabular}
\end{table}

\end{document}